\pdfoutput=1
\documentclass[11pt]{article}
\usepackage{graphicx} % Required for inserting images
\usepackage{diagbox}
\usepackage{authblk}
\usepackage{tikz}
\usetikzlibrary{calc}
\makeatletter
\renewcommand\AB@authnote[1]{\textsuperscript{#1}\hspace{5pt}}

\newcommand{\Opt}{\text{Opt}}
\makeatother
\usepackage{hyperref}
\usepackage{algorithm}
\usepackage{algpseudocode}
\usepackage{notation}
\usepackage{arxiv-2}
\usepackage{mathrsfs}
\usepackage{float}
\usepackage{setspace}
\usepackage{lmodern}        % Latin Modern 是 Computer Modern 的扩展字体（等宽支持更好）
\title{\normalfont Perturbing the Derivative: Wild Refitting for Model-Free Evaluation of Machine Learning Models under Bregman Losses\\
}
\author[1,3]{Haichen Hu\thanks{{Email: \texttt{huhc@mit.edu}}}} 
     \author[2,3]{David Simchi-Levi\thanks{{Email: \texttt{dslevi@mit.edu}}}}

\affil[1]{Center for Computational Science and Engineering, MIT} 
\affil[2]{Institute for Data, Systems, and Society, MIT} 
\affil[3]{Department of Civil and Environmental Engineering, MIT}

\date{}
\begin{document}
\maketitle
\begin{abstract}
    We study the excess risk evaluation of classical penalized empirical risk minimization (ERM) with Bregman losses. We show that by leveraging the idea of wild refitting, one can efficiently upper bound the excess risk through the so-called “wild optimism,” without relying on the global structure of the underlying function class. This property makes our approach inherently model-free. Unlike conventional analysis, our framework operates with just one dataset and black-box access to the training procedure. The method involves randomized Rademacher symmetrization and constructing artificially modified outputs by perturbation in the derivative space with appropriate scaling, upon which we retrain a second predictor for excess risk estimation. We establish high-probability performance guarantee under the fixed design setting, demonstrating that wild refitting under Bregman losses, with an appropriately chosen wild noise scale, yields a valid upper bound on the excess risk. Thus, our work is promising for theoretically evaluating modern opaque ML models, such as deep neural networks and generative models, where the function class is too complex for classical learning theory and empirical process techniques.
    
\end{abstract}
\vspace{-1cm}
\noindent\rule{\textwidth}{1pt} 
\vspace{0em}
\section{Introduction}
%\subsection{Background}
Deep Neural Networks, Generative AI, and Large Language Models (LLMs) have become central to modern industry, shaping applications across a wide range of domains, including business \citep{chen2023fiction,liang2025llm}, public governance \citep{androniceanu2024generative,acemoglu20258}, transportation \citep{zhang2024generative}, and healthcare \citep{bordukova2024generative}. As these systems increasingly influence critical decision-making processes, the question of how to evaluate their performance has become both practically and scientifically important. On the empirical side, the evaluation of these often opaque and hardly interpretable machine learning models has been extensively studied \citep{raschka2018model,shankar2024validates,mizrahi2024state,hendrycks2021measuring}. Researchers typically assess the effectiveness of deep learning models by running them on established benchmarks across diverse datasets, thereby demonstrating improvements in dimensions such as predictive accuracy \citep{he2016deep} and robustness to perturbations or distribution shifts \citep{liu2025comprehensive}. Such benchmarks provide a standardized and reproducible way to measure progress, and they have played a central role in driving empirical advances. In contrast, the theoretical understanding of how to rigorously evaluate these complex deep learning models remains very limited.

A major obstacle lies in the fact that the processes of training \citep{shrestha2019review,shen2024efficient}, including pre-training \citep{achiam2023gpt} and fine-tuning \citep{vm2024fine} involve optimization over millions or even billions of parameters \citep{brown2020language}. This extreme scale and complexity place these models well beyond the reach of classical tools in learning theory, such as the VC dimension \citep{abu1989vapnik}, covering number arguments \citep{zhou2002covering}, or the eluder dimension \citep{russo2013eluder}. While these theoretical measures have been foundational for understanding simpler hypothesis classes, they struggle to capture the behavior of highly complicated models trained with stochastic optimization at scale.

%In statistics, evaluating complex models often includes hold-out data splitting \citep{reitermanova2010data} and cross-validation \citep{refaeilzadeh2009cross,browne2000cross}. However, a key limitation of these methods is that such estimates reflect only the averaged risk over new samples instead of providing probabilistic guarantees on the realized risk of the predictor on the training dataset \citep{bates2024cross}. By contrast, applying these models in downstream decision-making, such as bandits and reinforcement learning, often requires high probability bounds on the excess risk,\citep{lattimore2020bandit,foster2023foundations}. 

Consequently, bridging these gaps between empirical practice and theoretical guarantees is, therefore, an open and pressing challenge for the field. A central challenge is that:
\begin{center}
\emph{Can we rigorously provide high probability bounds on the excess risk of complex ML models in theory without restrictive assumptions on the underlying function classes?}
\end{center}

\paragraph{Our Contribution} In this paper, we provide an affirmative answer to this question. Specifically, we study the most general empirical risk minimization (ERM) procedure under the Bregman loss as an abstraction of deep neural network training and develop an efficient algorithm for evaluating its excess risk in the fixed design setting. Our approach does not rely on structural assumptions about the underlying function family, which is why we call our method ``model-free'' in the title, but instead requires only black-box access to the training or optimization procedure, making it directly applicable to deep neural network training and LLM fine-tuning.

\paragraph{Paper Structure}
The remainder of our paper is organized as follows. In \cref{sec:model}, we introduce the empirical risk minimization (ERM) framework with Bregman loss and carefully define the quantities that are central in our analysis. In \cref{sec:algorithm}, we formally present our proposed method, Wild Refitting with Bregman Loss, and provide an intuitive explanation of the main ideas behind its construction. Building on this foundation, \cref{sec:theory_guarantee} and \cref{sec:bounding_hatr_n} develop the theoretical results, where we establish high-probability guarantees for our procedure and demonstrate the meaning of every component in these bounds. Together, these sections provide a comprehensive picture of both the algorithmic design and the statistical guarantees underlying our approach.

\paragraph{Notations}
We use $[n]$ to denote the set $\{1,2,\ldots,n\}$. $\cX^{m}$ is the product space of $m$ identical spaces $\cX$. For any convex function $\phi$, we denote its Bregman divergence by $D_{\phi}$. For two probability distributions $P$ and $Q$, we write $\text{KL}(P \,\|\, Q)$ for their Kullback–Leibler divergence and $H^2(P,Q)$ for their squared Hellinger distance. We write $\EE_X$ to denote expectation with respect to $X$. For any training algorithm $\cA$ and dataset $\cD$, $\cA(\cD)$ represents the predictor trained on the dataset $\cD$ through the procedure $\cA$.

\section{Related Works}
Our work is primarily related to the following streams of research: statistical learning and excess risk evaluation;data-splitting and resampling; empirical model evaluation and benchmarking.

\emph{Statistical learning} \citep{vapnik2013nature} has long served as a cornerstone in the theoretical analysis of machine learning algorithms. A central and active line of research focuses on understanding the \emph{excess risk} or generalization error of learning procedures. Classical approaches rely on empirical process theory to analyze ERM, with complexity measures such as the VC dimension \citep{vapnik2015uniform, blumer1989learnability}, covering numbers \citep{nickl2007bracketing, van2000empirical}, Rademacher complexity \citep{massart2007concentration, bartlett2005local}, and the fat-shattering dimension \citep{bartlett1994fat}. More recently, new notions such as the eluder dimension \citep{russo2013eluder}, eigendecay rate \citep{goel2017eigenvalue, li2024eigenvalue,hu2025contextual}, and sequential Rademacher complexity \citep{rakhlin2012relax} have been developed to study generalization in online learning. Despite these advances, bounding these metrics fundamentally depends on the global structure of the underlying function class. When the hypothesis space is extremely rich—such as when covering numbers are infinite—these methods break down and fail to yield meaningful excess risk guarantees \citep{adcock2021gap,kurkova2025networksfinitevcdimension}. In contrast, our approach is function-class free and circumvents these structural limitations. Recently, \citet{wainwright2025wild} proposed the idea of wild refitting to evaluate the mean-square error. We characterize the essential principles behind this idea and develop a substantially broader algorithmic framework, subsuming \citet{wainwright2025wild} as a special case.

In asymptotic statistics, the quality of estimation procedures is often evaluated through hold-out or \emph{sample-splitting} methods, where the dataset is divided into training and evaluation subsets\citep{reitermanova2010data,dobbin2011optimally}. While effective, these approaches suffer from inefficient data usage. Related techniques, such as cross-validation \citep{berrar2019cross, browne2000cross,refaeilzadeh2009cross, gorriz2024k}, incur a computational burden due to the need for retraining repeatedly. Moreover, a key limitation of these methods is that such estimates reflect only the averaged risk over new samples instead of providing probabilistic guarantees on the realized risk of the predictor on the training dataset \citep{bates2024cross}. By contrast, applying these models in downstream decision-making, such as bandits and reinforcement learning, often requires high probability bounds on the excess risk, \citep{lattimore2020bandit,foster2023foundations}.
An alternative is the use of \emph{resampling} methods, including bootstrap \citep{hesterberg2011bootstrap, diciccio1996bootstrap, davison1997bootstrap} and wild bootstrap \citep{mammen1993bootstrap, flachaire2005bootstrapping, davidson2010wild}. However, bootstrap is designed to approximate asymptotic properties such as confidence intervals (CIs), whereas we provide high-probability, non-asymptotic upper bounds on excess risk.

In deep learning, empirical evaluation of trained models is typically conducted through benchmarking, that is, testing their accuracy on standard datasets or comparing their performance against well-established baseline algorithms \citep{malaiya2019empirical,he2016deep,voloshin1empirical}. Recently, in the context of LLM research, thousands of such studies have been published. For instance, benchmarking on text tasks \citep{wang2018glue,Zhang2020BERTScore}, on reasoning and mathematical deduction \citep{tafjord2020proofwriter,amini-etal-2019-mathqa,clark2021transformers}, and on applications in business analytics, finance, and operations research \citep{huang2025orlm,xie2023pixiu,liang2025llm}. However, these efforts are purely empirical without a theoretical understanding of how well post-trained models truly perform. Our task in this paper is to address this gap.

\section{Empirical Risk Minimization with Bregman Loss}\label{sec:model}
In this section, we formally introduce the problem setup that serves as the foundation of our study. Throughout the discussion, we assume that the reader has a basic familiarity with standard concepts in convex analysis. For the sake of completeness, we also provide a concise overview of the mathematical background on convex analysis in \cref{app:preliminary}, which can be consulted as needed.

We consider the following empirical risk minimization (ERM) setting. We have an input space $\cX$, and an output space (prediction space) $\cY$, assuming that our prediction is scalar-valued, i.e., $\cY\subset \RR$. A predictor is defined as a mapping $f:\cX\mapsto\cY$ such that given any input $\cX$, it returns a prediction vector $f(x)$.

In this paper, \textbf{we focus on the fixed design setting} and present an upper bound on the excess risk. Specifically, we view $\cbr{x_i}_{i=1}^{n}$ as fixed, and $y_i,\ i=1:n$ are sampled independently from some unknown conditional distributions $\PP(\cdot|x_i),\ i=1:n$. The objective of our interest is
\[
f^*\in\argmin\cbr{\frac{1}{n}\sum_{i=1}^{n}\EE_{y_i}[\ell(y_i,f(x_i))]}.
\] %Correspondingly, in the random setting, our task is to find $f^*\in\argmin\cbr{\EE_{X,Y}[\ell(Y,f(X))]}.$

We focus on a Bregman loss function $\ell:\cY\times\cY\rightarrow\RR$, defined such that $\ell(x,y)$ corresponds to the Bregman divergence generated by a differentiable convex potential $\phi:\RR\mapsto\RR$, i.e.,
\[
\ell(x,y)=D_{\phi}(x,y).
\]
Intuitively, $\ell(y,f(x))$ serves as a measure of the discrepancy between the observed outcome and the prediction produced by a candidate function $f$. Directly optimizing over the space of all measurable functions is clearly infeasible. In practice, especially in deep learning, one typically restricts the search to an underlying function class, which we denote by $\cF$ and interpret as the underlying training architecture. Given a dataset $\cD=\cbr{(x_i,y_i)}_{i=1}^{n}$, the learner invokes a black-box empirical risk minimization (ERM) procedure $\Alg$ to approximately solve the following optimization problem:
\[
\hat{f}\in\argmin_{f\in\cF}\cbr{\frac{1}{n}\sum_{i=1}^{n}\ell(y_i,f(x_i))+\cR(f)},
\]
The term $\cR(f)$ serves as a regularization penalty imposed on the function $f$. Classical examples include ridge regression, where $\cR(f)$ corresponds to the squared $\ell_2$ norm, and Lasso, where it corresponds to the $\ell_1$ norm. In this work, for the sake of technical clarity, we adopt a simplified viewpoint: we assume that the trainer has prior knowledge ensuring that all reasonable predictions must lie within a compact set $\cC$. Accordingly, we model the regularizer $\cR(f)$ as the following:
\begin{align*}
\cR(f) &= 
\left\{
   \begin{array}{ll}
   \text{0}, & \text{if } f(x_i) \in \mathcal{\cC}\ \forall i, \\
   \infty,         & \text{otherwise.}
   \end{array}
\right.
\end{align*}
To proceed with the analysis, we introduce a structural assumption on the convex function $\phi$.
\begin{assumption}\label{ass:phi_function}
The function $\phi$ is $\beta$-smooth and $\alpha$-strongly convex.
\end{assumption}
This regularity condition ensures well-behaved curvature properties of $\phi$, which in turn yield desirable stability and identifiability features for the induced Bregman loss. In particular, under \cref{ass:phi_function}, the Bregman loss function satisfies the following four key properties. We defer their proofs to \cref{app:proofs_sec:model}.
\begin{proposition}\label{prop:beta_smooth}
    If $\phi$ is $\beta$-smooth, then $\forall\ y$, $D_{\phi}(\cdot,y)$ is $\beta$-smooth with respect to the first variable.
\end{proposition}
\begin{proposition}\label{prop:PL_ineq}
    If $\phi$ satisfies \cref{ass:phi_function}, then for any fixed $y$, $D_{\phi}(\cdot,y)$ satisfies the Polyak--\L{}ojasiewicz (PL) inequality with constant $\frac{\alpha^2}{\beta}$.
\end{proposition}
Moreover, we also assume that the Bregman loss function satisfies the quasi-triangle inequality.
\begin{proposition}\label{prop:quasi_triangle}
    Any function $\phi(u)$ which is $\alpha$-strongly convex and $\beta$-smooth satisfies \cref{ass:phi_function}. The corresponding $D_{\phi}(x,y)$ satisfies that
    \[
    D_{\phi}(x,y)^{1/2}\le C_0(D_{\phi}(x,z)^{1/2}+D_{\phi}(z,y)^{1/2}),\ C_0=\sqrt{\frac{\beta}{\alpha}}.
    \]
\end{proposition}
 It is worth emphasizing that \cref{ass:phi_function} is not a restrictive condition. In fact, it is readily satisfied in a wide range of statistical and machine learning models. To illustrate this, we now present several representative examples where the assumption holds naturally.
\begin{example}
In regression, $\phi(u)=\frac{1}{2}u_2^2$ satisfies \cref{ass:phi_function} with $\alpha=\beta=1$. Moreover, $D_{\phi}(x,y)=\frac{1}{2}(x-y)_2^2$ satisfies Proposition \ref{prop:quasi_triangle} with $C_0=1$.
\end{example}

\begin{example}
    In the stochastic binary classification setting, denoting the $1$ dimensional probability simplex as $\Delta_1\subset\RR^2$, and $\varepsilon_0\le p\le 1-\varepsilon_0$, to be the probability of predicting $1$ over $0$. For the convex function $\phi(p)=-\sqrt{p}-\sqrt{1-p}$, then we have that 
    \[
    \sqrt{2}\le\phi''(p)=\frac{1}{4p^{3/2}} + \frac{1}{4(1-p)^{3/2}}\le \frac{1}{2\varepsilon_0^{3/2}},
    \]
     and \cref{ass:phi_function} is satisfied. Moreover, the Bregman divergence is $$D_{\phi}(p_1,p_2)=\frac{(\sqrt{p_1}-\sqrt{p_2})^2}{2\sqrt{p_2}}+\frac{(\sqrt{1-p_1}-\sqrt{1-p_2})^2}{2\sqrt{1-p_2}}.$$
     The squared Hellinger distance between two Bernoulli random variables $$H^2(p_1,p_2)=\frac{1}{2}\rbr{(\sqrt{p_1}-\sqrt{p_2})^2+(\sqrt{1-p_1}-\sqrt{1-p_2})^2}.$$
    Therefore, we have \(\frac{1}{\sqrt{1-\varepsilon_0}}H(p_1,p_2)\le D^{1/2}_{\phi}(p_1,p_2)\le \frac{1}{\sqrt{\varepsilon_0}}H(p_1,p_2).\) By the fact that the Hellinger distance is a metric, we know that $D_{\phi}$ satisfies Proposition \ref{prop:quasi_triangle} with $C_0=\frac{\sqrt{1-\varepsilon_0}}{\sqrt{\varepsilon_0}}$.
\end{example}

\begin{example}
    In the density estimation setting, we consider the hypothesis class parametrized by $\cbr{p_{\theta}:\theta\in\Theta}$, $\Theta\subset\RR$ with loss function $\ell(\theta_1,\theta_2)=D_{\phi}(p_{\theta_1},p_{\theta_2})$, where $\phi(\theta)=\int\frac{1}{\log p_{\theta}(x)}p_{\theta}(x)dx$, then $D_{\phi}(p_{\theta_1},p_{\theta_2})=\text{KL}(p_{\theta_1}||p_{\theta_2})$. If we know that $p_{\theta}\ge \eta_0>0,\ \forall \theta$, then by \citet[pp.~132]{Polyanskiy_Wu_2025}, we have \[
    \sqrt{\log_2 e} H(p_{\theta_1}||p_{\theta_2})\le\sqrt{\text{KL}(p_{\theta_1}||p_{\theta_2})}\le \sqrt\frac{\log(\frac{1}{\eta_0}-1)}{1-2\eta_0}H(p_{\theta_1}||p_{\theta_2}).
    \]
    If for any $x,\theta$, the Hessian of the negative log-likelihood $-\frac{d^2\log(p_{\theta}(x))}{d\theta^2}\preceq \beta$, then the loss $\phi(\theta)$ is $\beta$-smooth. Moreover, we restrict our analysis to a sufficiently small neighborhood of 
$\theta^*$ such that there exists a constant $r_0 > 0$ s.t. $|\theta - \theta^*| \leq r_0,\ \forall\ \theta\in\Theta$. The Fisher information $I(\theta) = \mathbb{E}_{p_{\theta^*}}[s_\theta(X) s_\theta(X)]$ is positive with lower bound $\alpha>0$, then $\phi(\theta)$ is $\alpha$-strongly convex and the loss $\ell(\cdot,\cdot)$ satisfies \cref{ass:phi_function} and Proposition \ref{prop:quasi_triangle}.

\end{example}
%The first condition requires the loss function to be bounded, which is true as long as we are handling continuous functions on compact sets. It also says that when the prediction perfectly matches the true outcome, the loss will be $0$. The second condition ensures that the loss function exhibits metric-like behavior, including symmetry and the ``quasi''-triangle inequality. The third condition pertains to the smoothness and first-order stationarity of the loss function with respect to its first argument. It is important to note that we do not impose convexity on the loss function, which makes our condition. We provide some examples below to illustrate that many loss functions satisfy these conditions.

%Note that this is for the fixed design setting, where $\cbr{x_i}_{i=1}^{n}$ are all fixed. If we know that $(x_i,y_i)$ are i.i.d. samples from some unknown joint distribution $\PP_0$, then its randomized counterpart is defined as
%\[
%\cE(\hat{f})=\EE_{(X,Y)\sim\PP_{0}}[\ell(Y,\hat{f}(X))]-\EE_{(X,Y)\sim\PP_0}[\ell(Y,f^*(X))].
%\]

Returning back to our analysis, we first provide a theoretical explicit formula about $f^*$.
\begin{proposition}\label{prop:f*is_expectation}
    If $\phi$ is strictly convex, then we have
    \[
    f^*(x)=\EE[Y|X=x],\ \forall x\in\cX.
    \]
\end{proposition}
With Proposition \ref{prop:f*is_expectation}, in the remaining part of the paper, we will write $y=f^*(x)+w$, where $w$ is a zero-mean stochastic noise conditioned on $x$.  Correspondingly, in fixed design, we write $y_i=f^*(x_i)+w_i$. In this paper, for technical simplicity, we assume that $w$ has a symmetric distribution.
The performance metric of $\Alg$ is defined by the \emph{excess risk}:
\begin{align}\label{equa:excess_risk1}
\cE_{fix}(\hat{f}):=\EE_{y'_{1:n}}\sbr{\frac{1}{n}\sum_{i=1}^{n}\ell(y'_i,\hat{f}(x_i))-\frac{1}{n}\sum_{i=1}^{n}\ell(y'_i,f^*(x_i))}.
\end{align}
In equation \ref{equa:excess_risk1}, $x_i,\ i=1:n$ are fixed covariates in the dataset $\cD$ and  $y'_{i}\ i=1:n$ are new responses sampled from $\PP(\cdot|x_i)$ independent of the training process $\Alg$ and the training dataset $\cD$. 

By some algebra, we have that
\begin{align*}
    &\frac{1}{n}\sum_{i=1}^{n}\ell(y'_i,\hat{f}(x_i))-\frac{1}{n}\sum_{i=1}^{n}\ell(y'_i,f^*(x_i))\\
    =&\frac{1}{n}\sum_{i=1}^{n}\rbr{D_{\phi}(y'_i,\hat{f}(x_i))-D_{\phi}(y'_i,f^*(x_i))}\\
    =&\frac{1}{n}\sum_{i=1}^{n}\rbr{D_{\phi}(f^*(x_i),\hat{f}(x_i))+[\phi'(f^*(x_i))-\phi'(\hat{f}(x_i))]\cdot w'_i}.
\end{align*}
In the last inequality, we use the three-point equality of the Bregman divergence (Lemma \ref{lemma:three-point}). Treating $\cbr{x_i}_{i=1}^{n}$ as fixed, taking expectation with respect to new samples $y'_{1},\cdots,y'_n$, and noticing that $w'_i$ is independent of $\hat{f}$, we have that $\EE_{w_{1:n}}[(\phi'(f^*(x_i))-\phi'(\hat{f}(x_i)))\cdot w'_i]=0$, and
\[
\cE_{fix}(\hat{f})=\frac{1}{n}\sum_{i=1}^{n}D_{\phi}(f^*(x_i),\hat{f}(x_i)).
\]
%In random design, the risk measure is its expectation counterpart, \emph{excess risk}, which is
%$\cE(\hat{f}):=\EE_{X,Y}[\ell(Y,\hat{f}(X))]-\EE_{X,Y}[\ell(Y,f^*(X))]$.

Define the \emph{proxy distance} as
$$L_n(f_1,f_2):=\frac{1}{n}\sum_{i=1}^{n}\ell(f_1(x_i),f_2(x_i))=\frac{1}{n}\sum_{i=1}^{n}D_{\phi}(f_1(x_i),f_2(x_i)).$$
Intuitively, $\cE_{fix}(\hat{f})=L_n(f^*,\hat{f})$ captures the aggregate point-wise output discrepancies between the true optimizer $f^*$ and the estimator $\hat{f}$ with respect to the loss function. 

\section{Wild Refitting: Perturbation on the Derivatives}\label{sec:algorithm}
In this section, we provide a principled way to bound the excess risk in a function class-free manner via wild refitting. Wild refitting was first studied by \citet{wainwright2025wild} in the mean square loss setting, where the trainer cares about the point-wise square discrepancy between the two predictors. The intuition of wild refitting is to refit the trained predictor on a perturbed version of the original dataset and to retrain a new predictor using the same procedure. Wild refitting treats the predictor as a black box and directly operates on its output, enabling statistically valid inferences even without restrictive parametric assumptions on function classes regarding the training procedure.

Specifically, the key idea is that after we train the model $\hat{f}$ via the procedure $\Alg$ based on the dataset $\cD=\cbr{(x_i,y_i)}_{i=1}^{n}$, we compute the residual vector between the true outcome $y_i$ and our predicted value $\hat{f}(x_i)$. Then, we construct a sequence of Rademacher random variables $\cbr{\varepsilon_i}_{i=1}^{n}$ and multiply them by the residual vector sequence $\cbr{\tilde{w}_i}_{i=1}^{n}=\cbr{y_i-\hat{f}(x_i)}_{i=1}^{n}$ to obtain $\cbr{\varepsilon_i\cdot\tilde{w}_i}_{i=1}^{n}$. Subsequently, for some scale $\rho>0$, we consider the wild response sequence $y^\diamond_i:=G(\hat{f}(x_i),\rho\varepsilon_i\cdot\tilde{w}_i),\ i\in[n]$, where $G$ is a loss function-specific perturbation function. Finally, we compute the refitted wild solution $f^\diamond_{\rho}$ based on the artificially constructed dataset $\cbr{(x_i,y^\diamond_i)}_{i=1}^{n}$.

For square loss, one can directly perturb the original outputs for the wild responses by adding artificial stochastic noise to the prediction values \citep{wainwright2025wild}, i.e., $y^\diamond_i=\hat{f}(x_i)+\rho\varepsilon_i\tilde{w}_i$. However, for general Bregman losses, we need to perturb the derivative to control the fluctuations more precisely, as it is the derivative that dictates the change in function values. Specifically, for Bregman loss $D_{\phi}$, we denote $\phi^*$ as the Legendre–Fenchel conjugate of $\phi$, i.e.,
\[
\phi^*(y):=\sup_{x}\cbr{y\cdot x-\phi(x)}.
\]
Then, for some tuning parameter $\rho$, we define the wild responses $y_i^\diamond$ as the vector such that
\[
\phi'(y_i^\diamond):=\phi'(\hat{f}(x_i))-\rho\varepsilon_i\cdot\tilde{w}_i.
\]
By the Fenchel-Moreau Lemma \ref{lemma:Fenchel-Moreau} \citep{magaril2003convex}, this is equivalent to the following perturbing procedure:
\[
y_i^\diamond:=(\phi^*)'(\phi(\hat{f}(x_i))-\rho\varepsilon_i\tilde{w}_i).
\]
Under square loss where $\phi(x)=x^2$ and $\phi'(x)=2x$, output perturbation is tantamount to derivative perturbation, and the method of \citet{wainwright2025wild} falls within our framework as a special case.

\begin{remark}
One could also introduce a ``recentering'' function $\tilde{f}$ to construct the wild residuals $\tilde{w}_i = y_i - \tilde{f}(x_i)$. In practice, however, it is common to take $\tilde{f} = \hat{f}$. For simplicity, we therefore adopt this convention throughout and define the wild responses directly using $\hat{f}$, without introducing a separate recentering function.    
\end{remark}

\begin{algorithm}[ht]
\caption{\textbf{Wild-Refitting with Bregman Loss}}\label{alg:wild-refitting}
    \begin{algorithmic}
        \Require Procedure $\Alg$, training dataset $\cD_0=\cbr{(x_1,y_1),\cdots,(x_n,y_n)}$, noise scale $\rho>0$, loss function $\ell(x,y)=D_{\phi}(x,y)$, and refitting dataset $\cD_1=\emptyset$.
        \State Apply algorithm $\Alg$ on the training dataset. Get predictor $$\hat{f}=\Alg(\cbr{(x_i,y_i)}_{i=1}^{n}).$$
        \For{$i=1:n$}
        \State Compute residues $\tilde{w}_i=y_i-\hat{f}(x_i)$.
        \State Apply product between Rademacher random variable $\varepsilon_i$ and $\tilde{w}_i$:
        $$\tilde{w}^\diamond_i:=\varepsilon_i\cdot\tilde{w}_i.$$
        \State Construct wild responses 
        $$y^\diamond_i=(\phi^*)'(\phi'(\hat{f}(x_i))-\rho\tilde{w}^\diamond_i).$$
        \State Append $(x_i, y_i^\diamond)$ to $\mathcal{D}_1:$
        $$\mathcal{D}_1 \gets \mathcal{D}_1 \cup \{(x_i, y_i^\diamond)\}.$$
        \EndFor
        \State Compute the refitted wild solution $f^{\diamond}_{\rho}=\Alg(\cD_1)$.
        \State Output $\hat{f}$, $f^\diamond_\rho$, $\cD_1$, $\cD_0$.
    \end{algorithmic}
\end{algorithm}
Traditional generalization theory often requires taking the supremum over the entire model class. When the model class is highly complex, even local Rademacher complexity may become extremely large, leading to vacuous bounds. In contrast, our wild-refitting approach circumvents the need to analyze the Rademacher complexity of the model class. Instead, we control this term through the \emph{wild optimism}, provided by the outputs of \Cref{alg:wild-refitting}; see Lemma \ref{lemma:W_n<Opt} in \cref{sec:theory_guarantee} for details. In the next section, we will show that the excess risk could be efficiently upper bounded by the output of \cref{alg:wild-refitting} without structural assumptions on $\cF$.

\section{Bounding the Excess Risk in the Fixed Design}\label{sec:theory_guarantee}
Now, we present our theoretical guarantees for \cref{alg:wild-refitting}. Specifically, we establish an upper bound on the empirical excess risk under the fixed design. However, the theorem in \cref{sec:theory_guarantee} cannot be directly applied because it depends on some unknown quantity $\hat{r}_n$, which is intuitively the average discrepancy between the objective $\hat{f}$ and the ``best'' predictor $f^\dagger$ that we might achieve through the procedure $\Alg$. In \Cref{sec:bounding_hatr_n}, we derive a theorem that bounds $\hat{r}_n$, and further demonstrate a practical method for computing it.

 To bound the empirical excess risk, we first need to specify the quantities we want. In this paper, we use $\varepsilon$ and $w$ to represent the vectors $(\varepsilon_1,\cdots,\varepsilon_n)$ and $(w_1,\cdots,w_n)$. By the convexity of the Bregman loss with respect to the first variable, we have
\begin{align*}
\cE_{fix}(\hat{f})&=\frac{1}{n}\sum_{i=1}^{n}\ell(f^*(x_i),\hat{f}(x_i))\\
&\le \frac{1}{n}\sum_{i=1}^{n}\ell(y_i,\hat{f}(x_i))-\frac{1}{n}\sum_{i=1}^{n}\ell_1'(f^*(x_i),\hat{f}(x_i))w_i\\
&\le \frac{1}{n}\sum_{i=1}^{n}\ell(y_i,\hat{f}(x_i))+\abr{\frac{1}{n}\sum_{i=1}^{n}\ell_1'(f^*(x_i),\hat{f}(x_i))w_i}.
\end{align*}
where $(x_i,y_i),\ i=1:n$ are the data points in the training dataset $\cD$, and $w_i=y_i-f^*(x_i)$. 
 
%\[
%\cE_{\cD}(\hat{f})\le \frac{1}{n}\sum_{i=1}^{n}\ell(y_i,\hat{f}(x_i))+2\abr{\frac{1}{n}\sum_{i=1}^{n}(\phi'(f^*(x_i))-\phi'(\hat{f}(x_i)))w_i}.
%\]
The first term is the training error term, which is known to us. Our task is thus to bound the absolute value of the second one. We define the following term as \emph{true optimism complexity}.
\[
\Opt^*(\hat{f}):=\frac{1}{n}\sum_{i=1}^{n}[\phi'(f^*(x_i))-\phi'(\hat{f}(x_i))]w_i=\frac{1}{n}\sum_{i=1}^{n}\ell_1'(f^*(x_i),\hat{f}(x_i))w_i.
\]
Therefore, we only need to bound $|\Opt^*(\hat{f})|$ for the excess risk bound.

Analyzing $\Opt^*(\hat{f})$ usually involves empirical process theory, where we try to bound the optimism term in terms of the sample size $n$ and some number based on the assumptions about the training architecture and function family. In this paper, we define the following empirical process $W_n(r)$ as \emph{wild noise complexity}:
\[
W_n(r):r\mapsto\sup_{f\in\cB_r(\hat{f})} \cbr{\frac{1}{n}\sum_{i=1}^{n}[\phi'(f(x_i))-\phi'(\hat{f}(x_i))]\varepsilon_i\tilde{w}_i},
\]
where for any $g$, the term $\cB_r(g)$ is defined as the empirical Bregman ball:
\[
\cB_r(g):=\cbr{f\in\cF|L_n(g,f)\le r^2}=\cbr{f\in\cF|\sqrt{L_n(g,f)}\le r}.
\]
However, in many applications, the function class $\cF$ is too complicated such that analyzing the structure of it is intractable. To evaluate such models theoretically, we utilize the wild refitting procedure and define the following \emph{wild optimism} in terms of quantities we learn from \cref{alg:wild-refitting}:
\[
\widetilde{\Opt}^{\diamond}(f^{\diamond}_{\rho}):=\frac{\beta}{\rho n}\sum_{i=1}^{n}\ell(\hat{f}(x_i),f^{\diamond}_{\rho}(x_i))+\frac{\beta}{\rho n}\sum_{i=1}^{n}\ell(\hat{f}(x_i),y_i^\diamond)-\frac{\alpha^2}{\beta\rho n}\sum_{i=1}^{n}\ell(y^{\diamond}_i,f^{\diamond}_{\rho}(x_i)).
\]
In formality, the wild optimism involves the empirical difference between $\hat{f}$ and $f_\rho^\diamond$; the average perturbation magnitude from $\hat{f}(x_i)$ to $y_i^\diamond$; and the training error of $f_\rho^\diamond$.

The power of the wild-refitting is illustrated in the following lemma.
\begin{lemma}\label{lemma:W_n<Opt}
    \[
  W_n(\frac{1}{n}\sum_{i=1}^{n}\ell(\hat{f}(x_i),f^{\diamond}_{\rho}(x_i)))\le \frac{\beta}{\rho n}\sum_{i=1}^{n}\ell(\hat{f}(x_i),f^{\diamond}_{\rho}(x_i))+\frac{\beta}{\rho n}\sum_{i=1}^{n}D_{\phi}(\hat{f}(x_i),y_i^\diamond)-\frac{\alpha^2}{\beta\rho n}\sum_{i=1}^{n}\ell(y^{\diamond}_i,f^{\diamond}_{\rho}(x_i)),
  \]
  i.e.,
  \[
  W_n(\sqrt{L_n(\hat{f},f^\diamond_{\rho})})\le \widetilde{\Opt}^\diamond(f^\diamond_\rho).
  \]
\end{lemma}
Lemma \ref{lemma:W_n<Opt} shows that we could upper bound the supremum of the empirical process by the wild optimism term, which is something that we know from the outputs of \cref{alg:wild-refitting}. Roughly speaking, this property ensures that the maximal argument in the empirical process $W_n(\cdot)$ under some radius is explicitly found at the point $f_\rho^\diamond$, which helps us avoid the complexity of the function class and is vital for our analysis. See \cref{graph:lemma_illust} for an intuitive geometrical illustration.
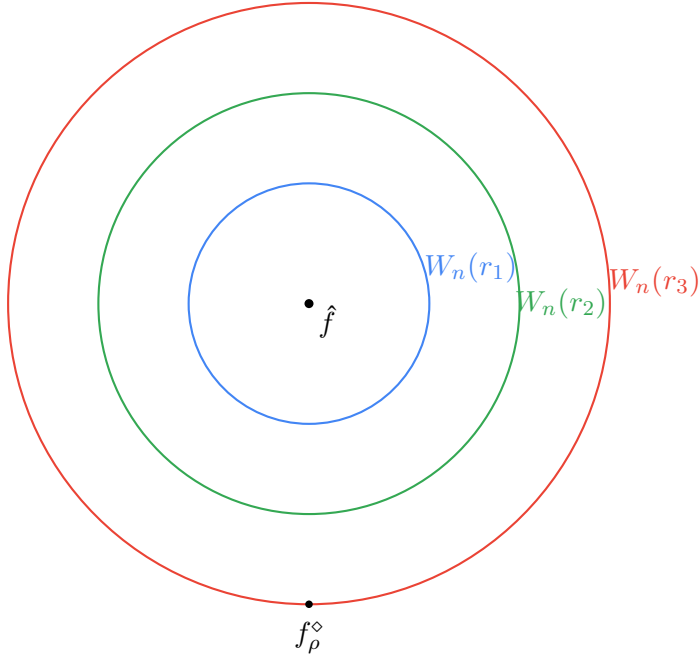
\begin{figure}[htbp]
\begin{center}
\begin{tikzpicture}[scale=2]

  % Define colors
  \definecolor{circleblue}{RGB}{66, 133, 244}
  \definecolor{circlegreen}{RGB}{52, 168, 83}
  \definecolor{circlered}{RGB}{234, 67, 53}

  % Define true radii (r1 < r2 < r3)
  \def\rone{0.8}   % smallest
  \def\rtwo{1.4}
  \def\rthree{2.0} % largest

  % Draw concentric circles (inner → outer)
  \draw[thick, circleblue] (0,0) circle (\rone);
  \draw[thick, circlegreen] (0,0) circle (\rtwo);
  \draw[thick, circlered] (0,0) circle (\rthree);

  % Center point
  \fill (0,0) circle (0.03);
  \node at (0.12, -0.12) {$\hat{f}$};

  % Labels (place each near its circle, avoiding overlap)
  \node[circleblue] at (\rone +0.277, 0.24) {$W_n(r_1)$};
  \node[circlegreen] at (\rtwo + 0.277, 0) {$W_n(r_2)$};
  \node[circlered] at (\rthree + 0.3, 0.15) {$W_n(r_3)$};

  % f_rho^diamond point on outer circle (top-right)
  \def\angle{40}
  \coordinate (fdiamond) at (0, -\rthree);
  \fill (fdiamond) circle (0.025);
  \node[anchor=north] at ($(fdiamond)+(0,-0.05)$) {$f_{\rho}^\diamond$};

  % Full formula below the point
  \node[align=left, anchor=north, scale=1] 
    at ($(fdiamond)+(0,-0.4)$)
    {$\widetilde{\text{Opt}}^\diamond(f_{\rho}^\diamond) \ge \arg\max_{\mathcal{B}_{r_3}(\hat{f})} 
    \left\{ \frac{1}{n} \sum_{i=1}^{n} 
    \left[ \phi'(f(x_i)) - \phi'(\hat{f}(x_i)) \right] 
    \varepsilon_i \tilde{w}_i \right\}$};

\end{tikzpicture}
\end{center}
\caption{Illustration of Lemma \ref{lemma:W_n<Opt}}
\label{graph:lemma_illust}
\end{figure}

We further define the noiseless optimization problem and its solution as follows:
\[
  f^\dagger=\argmin_{f\in\cF}\cbr{\frac{1}{n}\sum_{i=1}^{n}\ell(f^*(x_i),f(x_i))+\cR(f)}.
  \]
Intuitively, $f^\dagger$ is the ``optimal'' solution that we can get because its training dataset is the cleanest without noise. Then, we define intermediate \emph{noiseless optimism} as 
 \[
 \text{Opt}^\dagger(\hat{f})=\frac{1}{n}\sum_{i=1}^{n}\ell'_1(f^\dagger(x_i),\hat{f}(x_i))w_i,
 \]
 and the empirical process:
 \[
 Z_n^{\varepsilon}(r):=\sup_{f\in\cB_r(f^\dagger)}\sbr{\frac{1}{n}\sum_{i=1}^{n}\ell'_1(f^\dagger(x_i),\hat{f}(x_i))\varepsilon_i|w_i|}.
 \]
Intuitively, $f^\dagger$ is the predictor we get through procedure $\Alg$ if there is no noise in our dataset $\cD$. Defining $\hat{r}_n:=\sqrt{L_n(f^\dagger,\hat{f})}$ to be the empirical discrepancy between $f^\dagger$ and $\hat{f}$, we have the following theorem.
\begin{theorem}\label{thm:main}
    Consider any radius $r$ such that $r\ge \sqrt{L_n(f^\dagger,\hat{f})}$, and let $\rho>0$ be the noise scale for which $\sqrt{L_n(\hat{f},f^\diamond_{\rho})}=3(\beta/\alpha)^{1/2}r$. Then, for any $0<\delta<1$, with probability at least $1-8\delta$,\[
    |\Opt^*(\hat{f})|\le|\widetilde{\Opt}^{\diamond}(f^{\diamond}_{\rho})|+A_n(\hat{f})+\sbr{\sqrt{L_n(f^*,f^\dagger)}+5r}\cdot\frac{2||w||_{\infty}(\beta^{3/2}\vee\beta^2)\sqrt{\log(1/\delta)}}{(\alpha^{3/2}\wedge\alpha)\sqrt{n}}.
    \]
    Therefore, regarding the excess risk, with probability at least $1-8\delta$,
    \[
    \cE_{\cD}(\hat{f})\le \frac{1}{n}\sum_{i=1}^{n}\ell(y_i,\hat{f}(x_i))+\cbr{|\widetilde{\Opt}^{\diamond}(f^{\diamond}_{\rho})|+A_n(\hat{f})+\sbr{\sqrt{L_n(f^*,f^\dagger)}+5r}\cdot\frac{2||w||_{\infty}(\beta^{3/2}\vee\beta^2)\sqrt{\log(1/\delta)}}{(\alpha^{3/2}\wedge\alpha)\sqrt{n}}}.
    \]
    The term $A_n(\hat{f})$ is called \emph{pilot error}, which is given by
    \[
    A_n(\hat{f}):=\sup_{f\in\cB_{3r(\beta/\alpha)^{1/2}}(\hat{f})}\frac{1}{n}\sum_{i=1}^{n}\ell_1'(\hat{f}(x_i),f(x_i))\varepsilon_i\cdot(\hat{f}(x_i)-f^*(x_i)).
    \]
    Moreover, we name the probabilistic deviation term as $B_n(t)$.
    \[
    B_n(t):=\cbr{\sqrt{L_n(f^*,f^\dagger)}+5r}\frac{2||w||_{\infty}(\beta^{3/2}\vee\beta^2)t}{(\alpha^{3/2}\wedge\alpha)\sqrt{n}}.
    \]
\end{theorem}
See \cref{subsec:proof of mainthm} for the proof of this theorem.

We now provide some intuition behind this theorem. In essence, it states that the absolute value of the true optimism, $|\Opt^*(\hat{f})|$, can be bounded from above by the sum of three components: the wild optimism term, the probability deviation term, and the pilot term.

For the probability deviation term, we comment that it gradually vanishes as the sample size $n\rightarrow\infty$, the term $\sqrt{L_n(f^*,f^\dagger)}$ represents the average discrepancy between $f^\dagger$ and $f^*$, accounting for potential model mis-specification in the probability deviation term.

For the pilot error term, we notice that by Lemma \ref{lemma:W_n<Opt}, we can bound $W_n(3(\beta/\alpha)^{1/2}r)$ by the wild optimism term $\widetilde{\Opt}^\diamond(f^\diamond_{\rho})$, i.e.,
\[
\widetilde{\Opt}^\diamond(f^\diamond_{\rho})\ge W_n(3(\beta/\alpha)^{1/2}r)=\sup_{f\in\cB_{3r\sqrt{\beta/\alpha}}(\hat{f})} \cbr{\frac{1}{n}\sum_{i=1}^{n}\ell_1'(\hat{f}(x_i),f(x_i))\varepsilon_i\cdot\tilde{w}_i}.
\]
Comparing the right-hand side with the definition of the pilot error term, we see that the only difference is that $\tilde{w}_i$ in $W_n(3(\beta/\alpha)^{1/2}r)$ is replaced by $(\hat{f}(x_i)-f^*(x_i))$ in $A_n(\hat{f})$. Intuitively, as long as the training procedure is well-behaved, the discrepancy between $\hat{f}(x_i)$ and $f^*(x_i)$ is primarily expected to be smaller in magnitude than that plus the fluctuation generated by the additional stochastic noise term $w_i$. This observation suggests that the pilot error term can be controlled by $W_n(3(\beta/\alpha)^{1/2}r)$, which can itself be further bounded by $\widetilde{\Opt}^\diamond(f^\diamond_{\rho})$. Such a comparison is natural as it reflects the intuition that the estimation error between $\hat{f}$ and $f^*$ is typically less volatile than the combination of that error with random noise.

Moreover, we do not require noises to have the same distribution, which allows heteroskedasticity among the noises. Indeed, in our bound, the only term is the maximal amplitude among noises.

\section{Bounding the Noiseless Estimation Error}\label{sec:bounding_hatr_n}
From the conditions in \Cref{thm:main}, it is clear that the result becomes meaningful only if we can obtain an upper bound on $\hat{r}_n=\sqrt{L_n(f^\dagger,\hat{f})}$. In this section, we present an explicit theorem for such bounds, under the condition that the training procedure $\Alg$ satisfies the non-expansive property (\cref{def:non_expansive}). See \Cref{app:proof of bound_hat_rn} for the full proof.
\begin{definition}[$\phi$-non-expansive]\label{def:non_expansive}
    We say the training procedure $\Alg$ is $\phi$-non-expansive if, for any noiseless dataset $\cD_u=\cbr{(x_i,f^*(x_i))}_{i=1}^{n}$ and the noisy dataset $\cD_n=\cbr{(x_i,f^*(x_i)+u_i)}_{i=1}^{n}$, denoting $f^\dagger=\cA(\cD_u)$ and $\tilde{f}=\cA(\cD_n)$ as the predictors trained on the two datasets, then we have
    \[
    L_n(f^\dagger,\tilde{f})\le \frac{1}{n}\sum_{i=1}^{n}\ell'_1(f^\dagger(x_i),\tilde{f}(x_i))u_i=\frac{1}{n}\sum_{i=1}^{n}[\phi'(f^\dagger(x_i))-\phi'(\tilde{f}(x_i))]u_i
    \]
\end{definition}
\begin{theorem}\label{thm:bound_hat_rn}
    If our training procedure $\Alg$ is $\phi$-non-expansive, then for any $\delta\le e^{-9}$, we have that with probability at least $1-4\delta$,
    \[
    \hat{r}_n^2\le \max\cbr{\frac{\log^2(1/\delta)}{n},W_n\rbr{(2+\frac{1}{\log(1/\delta)})\hat{r}_n}}+\hat{r}_n^2\frac{6||w||_{\infty}\beta^{3/2}}{\alpha \sqrt{\log(1/\delta)}}+A_n(\hat{f}).
    \]
    Moreover, if the underlying function class $\cF$ is convex, then with probability at least $1-4\delta$, for any noise scale $\rho$,
    \[
    \hat{r}_n^2\le \max\cbr{(r^\diamond_\rho)^2,\frac{\log(1/\delta)^2}{n},\frac{\hat{r}_n}{r^\diamond_\rho}W_n\rbr{\sqrt{\beta/\alpha}\rbr{2+\frac{1}{\sqrt{\log(1/\delta)}}}\hat{r}_n}}+\hat{r}_n^2\frac{6||w||_{\infty}\beta^{3/2}}{\alpha \sqrt{\log(1/\delta)}}+A_n(\hat{f}).
    \]
\end{theorem}
\paragraph{Illustration of \Cref{thm:bound_hat_rn}:}
The first bound in \Cref{thm:bound_hat_rn} involves several components, including the probability deviation terms and the pilot error term. The latter is dominated by the wild optimism term $\widetilde{\Opt}^{\diamond}(f^\diamond_{\rho})$ for a suitable choice of $\rho > 0$. When $\delta$ is sufficiently small, the probability deviation term is expected to become negligibly small, leaving the main contribution as 
\[
\max\left\{\tfrac{\log^2(1/\delta)}{n},\, W_n\rbr{(2+\frac{1}{\log(1/\delta)})\hat{r}_n}\right\}.
\]
Moreover, the term $\log^2(1/\delta)/n$ is of lower order provided that $\hat{r}_n \gtrsim 1/\sqrt{n}$ \citep{bartlett2005local}, which corresponds to the simplest parametric function class setting. Consequently, our goal is, roughly speaking, to identify a smallest $r$ satisfying
\[
r^2 \geq W_n\rbr{(2+\frac{1}{\log(1/\delta)})r},
\]
For the chosen value of $t$. For any $q$, evaluating $W_n(q)$ is tractable since it only depends on the predictor $\hat{f}$ and the wild noises $\cbr{\tilde{w}_i}_{i=1}^{n}$. Alternatively, by Lemma \ref{lemma:W_n<Opt}, one can upper bound $W_n(q)$ by varying the noise scale $\rho$ until the corresponding wild predictor $f^\diamond_{\rho}$ satisfies $L_n(\hat{f},f^\diamond_{\rho}) = q^2$.

Computationally, we examine the expression of $W_n(r)$. By definition, equivalently, we just need to solve the optimization problem:
$\min_{f\in\cB_r(\hat{f})}\cbr{\frac{1}{n}\sum_{i=1}^{n}\phi'(f(x_i))\varepsilon_i\cdot\tilde{w}_i}.$

 As long as the set $\cB_r(\hat{f})$ and $\cU=\cbr{(f(x_1),\cdots,f(x_n)): f\in\cB_r(\hat{f})}$ is convex, the set $$\cU_{\phi}:=\cbr{(\phi'(f(x_1)),\cdots,\phi'(f(x_n))):f\in\cB_r(\hat{f})}$$
is also convex. Hence, the optimization problem is reduced to a linear optimization problem over a convex set, which can be solved efficiently \citep{martin2012large,NEURIPS2021_a8fbbd3b}. 

The second bound of \Cref{thm:bound_hat_rn} is a bit looser but more interpretable for convex classes because we do not need to numerically solve $r^2 \geq W_n\rbr{(2+\frac{1}{\log(1/\delta)})r}$, which might require us to solve iteratively, and every iteration requires solving a linear programming problem. Instead, we have
\[
\hat{r}_n\le \max\cbr{r^\diamond_{\rho},\frac{W_n\rbr{\sqrt{\beta/\alpha}\rbr{2+\frac{1}{\sqrt{\log(1/\delta)}}}r^{\diamond}_{\rho}}}{r^{\diamond}_{\rho}}}.
\]
This bound enables us to directly compute an upper bound for the noise scale $\rho$ to obtain an upper bound of $\hat{r}_n$. Then we could adjust our noise scale accordingly and refine our estimation of $\hat{r}_n$.

Once we have obtained an upper bound on $\hat{r}_n$, we can then return to \Cref{thm:main}. It suggests that—apart from the higher-order and pilot terms—we can upper bound the true optimism $\Opt^*(\hat{f})$  via the wild optimism $\widetilde{\Opt}^\diamond(f^\diamond_\rho)$ for the noise scale that we have chosen. Finally, we combine this upper bound on the optimism in \Cref{thm:main} to upper bound the excess risk.
\section{Discussion}
In this paper, we propose a wild refitting procedure for estimating the model excess risk under general Bregman losses. Our algorithm recovers the result in \citet{wainwright2025wild}, and reveals that the true fundamental principle of wild refitting is perturbing in the derivative space.

Our excess risk estimation bound does not rely on the global structure or prior knowledge of the underlying function class; instead, we just assume black-box access to the procedure, which makes it especially suitable for evaluating the performance of deep neural networks and fine-tuned LLMs, where the training architectures and the underlying function classes are too complicated to analyze. 

Finally, we emphasize that wild refitting remains a new approach for model evaluation, and several important open questions deserve further exploration. First, we restrict our analysis to fixed design setting and whether it is possible to derive similar excess risk bound for random design is an important open question. Second, in our analysis, the excess risk bound remains dependent on the local structure around $\hat{f}$, specifically the pilot error term $A_n(\hat{f})$. The possibility of rigorously bounding this term remains an open question. It is also natural to ask whether wild refitting can be extended to more settings, such as high-dimensional set-valued estimators and predictions with other regularization penalties. Finally, it would be valuable to conduct empirical studies to evaluate real-world trained AI models, thereby demonstrating the practical effectiveness of wild refitting.

\clearpage
\bibliographystyle{plainnat}
\bibliography{haichen/sections/refs}
\appendix
\clearpage
\section{Preliminary Concepts in Convex Analysis}\label{app:preliminary}
The introduction in this section holds for general convex mappings from $\RR^d$ to $\RR$. In our main text, the convex function $\phi$ maps from $\RR$ to $\RR$ as a special case.
\begin{definition}[$\beta$-smoothness]
A differentiable function $f:\mathbb{R}^d \to \mathbb{R}$ is called \emph{$\beta$-smooth} 
if its gradient is $\beta$-Lipschitz continuous, i.e.,
\[
\|\nabla f(x) - \nabla f(y)\| \leq \beta \|x-y\|, \quad \forall x,y \in \mathbb{R}^d.
\]
This is equivalent to
\[
f(y) \leq f(x) + \langle \nabla f(x), y-x \rangle + \tfrac{\beta}{2}\|y-x\|^2.
\]
\end{definition}

\begin{definition}[$\alpha$-strong convexity]
A differentiable function $f:\mathbb{R}^d \to \mathbb{R}$ is called \emph{$\alpha$-strongly convex} 
if there exists $\alpha > 0$ such that
\[
f(y) \geq f(x) + \langle \nabla f(x), y-x \rangle + \tfrac{\alpha}{2}\|y-x\|^2,
\quad \forall x,y \in \mathbb{R}^d.
\]
\end{definition}

\begin{definition}[Polyak--Łojasiewicz (PL) inequality]
A differentiable function $f:\mathbb{R}^d \to \mathbb{R}$ is said to satisfy the 
\emph{$\mu$-Polyak--Łojasiewicz (PL) inequality} if there exists $\mu > 0$ such that
\[
\frac{1}{2}\|\nabla f(x)\|^2 \geq \mu \big(f(x) - f(x^*)\big), 
\quad \forall x \in \mathbb{R}^d,
\]
where $x^* \in \arg\min_{z \in \mathbb{R}^d} f(z)$.
\end{definition}
\begin{definition}[Bregman divergence]
Let $\phi:\mathbb{R}^d \to \mathbb{R}$ be a differentiable and convex function.  
The \emph{Bregman divergence} associated with $\phi$ is defined as
\[
D_{\phi}(x,y) \;=\; \phi(x) - \phi(y) - \langle \nabla \phi(y), \, x-y \rangle,
\quad \forall x,y \in \mathbb{R}^d.
\]
\end{definition}

\section{Supporting Lemmas}\label{app:supporting_lemmas}
\begin{lemma}[Three-point equality of Bregman loss]\label{lemma:three-point}
For any differentiable function $\phi$, the corresponding Bregman divergence satisfies
\[
D_{\phi}(x,z)=D_{\phi}(x,y)+D_{\phi}(y,z)+\inner{\phi'(y)-\phi'(z)}{x-y},\ \forall\ x,y,z. 
\]
\end{lemma}
The proof of Lemma \ref{lemma:three-point} can be done simply by checking the definition of the Bregman divergence.
\begin{lemma}\label{lemma:PL_ineq}\citep[Proposition~2.7, p.~14]{chewi2025lectures}
    Let $f:\RR^d\mapsto \RR$ and $\alpha>0$. The following implications hold.
    \begin{itemize}
        \item[1.] $f$ is $\alpha$ strongly convex then, $f$ satisfies the Polyak-\L{}ojasiewicz inequality with constant $\alpha>0$.
        \item[2.] If $f$ satisfies P\L{} inequality with constant $\alpha>0$, then it satisfies that
        \[
        f(x)-f^*\ge\frac{\alpha}{2}\inf_{x^*\in\cX^*}||x-x^*||_2^2,
        \]
        where $\cX^*$ denotes the set of minimizers of $f$.
    \end{itemize}
\end{lemma}
\begin{lemma}[Fenchel-Moreau Theorem \citep{magaril2003convex}]\label{lemma:Fenchel-Moreau}
    Let $f:\mathbb{R}^d\to(-\infty,+\infty]$ be proper, convex, and lower-semicontinuous, and define its convex conjugate
\[
f^*(y)\;=\;\sup_{x\in\mathbb{R}^d}\{\langle y,x\rangle-f(x)\}.
\]
Denoting $\partial f(s)$ as the subgradient set of $f$ at $s$, then we have $y\in\partial f(x)\Longleftrightarrow x\in\partial f^*(y)$.
\end{lemma}
\begin{lemma}[Hoeffding Inequality]\label{lemma:Hoeffding}
    If $X_1,\cdots,X_n$ are independent and satisfy $X_i\in[a_i,b_i]$, then for any $t>0$,
    \[
\Pr\!\left( \left| \frac{1}{n}\sum_{i=1}^n X_i - \mathbb{E}\!\left[\frac{1}{n}\sum_{i=1}^n X_i\right] \right| \ge t \right)
\;\le\; 2 \exp\!\left( -\frac{2n^2 t^2}{\sum_{i=1}^n (b_i-a_i)^2} \right).
\]

\end{lemma}
\begin{lemma}[Concentration Inequality about bounded r.v. with Lipschitz function]\label{lemma:concentration_Lip}\citep[Thm 3.24]{wainwright2019high}
Consider a vector of independent random variables $(X_1,\cdots,X_n)$, each taking values in $[0,1]$, and let $f:\RR^n\mapsto\RR$ be a convex, $\ell$-Lipschitz with respect to the Euclidean norm. Then for all $t\ge0$, we have
\[
\PP(|f(X)-\EE[f(X)]|\ge t)\le 2e^{\frac{-t^2}{2L^2}}.
\]
\end{lemma}
\begin{lemma}\label{lemma:triangle_Ln}
    If the loss function $\ell$ satisfies Proposition \ref{prop:quasi_triangle}, then the empirical divergence $L_n$ indicates that for any function $f,g,h$, 
    \[
    (L_n(f,g))^{1/2}\le \sqrt{2}
    C_0\rbr{L_n^{1/2}(f,h)+L^{1/2}_n(h,g)}.
    \]
\end{lemma}
\begin{proof}[Proof of Lemma \ref{lemma:triangle_Ln}]
    By definition, we have 
    \begin{align*}
        L_n(f,g)=&\frac{1}{n}\sum_{i=1}^{n}\ell(f(x_i),g(x_i))\\
        \le& \frac{1}{n}\sum_{i=1}^{n}\sbr{C_0\rbr{\sqrt{\ell(f(x_i),h(x_i))}+\sqrt{\ell(h(x_i),g(x_i))}}}^2\\
        \le&2C_0^2\rbr{\frac{1}{n}\sum_{i=1}^{n}\ell(f(x_i),h(x_i))+\frac{1}{n}\sum_{i=1}^{n}\ell(h(x_i),g(x_i))}\\
        =&2C_0^2(L_n(f,h)+L_n(h,g)).
    \end{align*}
    Taking the square root of both sides and noticing that $\sqrt{a+b}\le \sqrt{a}+\sqrt{b}$, we finish the proof.
\end{proof}
\iffalse
\begin{lemma}\citep{JMLR:v6:elisseeff05a}\label{lemma:stability_concentration}
    Suppose some algorithm $\cA$ satisfies pointwise stability $r$ in \Cref{def:stability} with respect to loss function $\ell$. The loss function $\ell(\cdot,\cdot)$ is bounded between $0$ and $M>0$. If we are given a dataset $\cD=\cbr{(x_i,y_i)}_{i=1}^{n}$ where $(x_i,y_i)\overset{\text{i.i.d.}}{\sim}\PP_{X,Y}$, then for any $1>\delta>0$, with probability at least $1-\delta$, 
    \[
    \EE_{X,Y}[\ell(Y,\hat{f}(X))]\le \frac{1}{n}\sum_{i=1}^{n}\ell(y_i,\hat{f}(x_i))+\sqrt\frac{M^2+12Mn\beta}{2n\delta}.
    \]
    The expectation is taken over the joint distribution $\PP_{X,Y}$.
\end{lemma}
\begin{lemma}\citep{charles2018stability}\label{lemma:stability_constant}
    If for any $u$ fixed, the function $\ell(\cdot,u)$ is $\ell$-Lipschitz continuous, the functional $T_n(g)=\frac{1}{n}\sum_{i=1}^{n}\ell(y_i,g(x_i))$ satisfies the Polyak--Łojasiewicz (PL) inequality with parameter $\mu$. If the learning algorithm $\cA$ satisfies that $\hat{f}=\argmin_{g\in\cX}T_n(g)$ for some set $\cX$, where $\hat{f}$ is trained by $\cA$ on dataset $\cD=\cbr{(x_i,y_i)}_{i=1}^{n}$. Then, $\cA$ has pointwise hypothesis stability with $\epsilon_{sta}=\frac{2L^2}{\mu(n-1)}$.
\end{lemma}
\fi
%\section{Omitted Proofs}\label{app:omitted proofs}
\section{Proofs in \cref{sec:model}}\label{app:proofs_sec:model}
\begin{proof}[Proof of Proposition \ref{prop:beta_smooth}]
    We prove this by direct calculation, for any $y$ fixed,
    \[
    D_{\phi,1}'(x,y)=\phi'(x)-\phi'(y).
    \]
    Therefore, \[
    |D_{\phi,1}'(x_1,y)-D_{\phi,1}'(x_2,y)|=|\phi'(x_1)-\phi'(x_2)|\le \beta|x_1-x_2|.
    \]
    So we finish the proof.
\end{proof}
\begin{proof}[Proof of Proposition \ref{prop:PL_ineq}]
    Denoting $x^*$ as the unique minimizer of $\phi$, from the $\beta$ smoothness, we know that
    \[
    D_{\phi}(x,y)\le \frac{\beta}{2}(x-y)^2.
    \]
    On the other hand, for $g_{\phi,y}(x)=D_{\phi}(x,y)$,
    \[
    g_{\phi,y}'(x)=\phi'(x)-\phi'(y),
    \]
    The unique minimizer of $g_{\phi,y}(x)$ is $x^*=y$, with $g_{\phi,y}(y)=0$. By strong convexity, we know that
    \[
    (\phi'(x)-\phi'(y))(x-y)\ge \alpha(x-y)^2;
    \]
    then, we have that if $x>y$,
    \[
    \phi'(x)-\phi'(y)\ge \alpha (x-y).
    \]
    Combining all these parts together, we have
    \begin{align*}
        (D_{\phi,1}'(x,y))^2= (\phi'(x)-\phi'(y))^2\ge\alpha^2(x-y)^2\ge \frac{2\alpha^2}{\beta}D_{\phi}(x,y)=\frac{2\alpha^2}{\beta}(D_{\phi}(x,y)-D_{\phi}(x^*,y)).
    \end{align*}
    Setting the constant to $\frac{\alpha^2}{\beta}$, we finish the proof.
\end{proof}
\begin{proof}[Proof of Proposition \ref{prop:quasi_triangle}]
    By the definition of $\beta$-smoothness and $
    \alpha$-strong convexity, we know that
    \[
    \frac{\alpha}{2}(x-y)^2 \le D_{\phi}(x,y)\le \frac{\beta}{2}(x-y)^2.
    \]
    Therefore, we know that $\frac{\sqrt{\alpha}}{\sqrt{2}}||x-y||_2\le\sqrt{D_{\phi}(x,y)}\le \frac{\sqrt{\beta}}{\sqrt{2}}||x-y||_2$.
    Hence, by the fact that $\ell=D_{\phi}$,
    \[
    \sqrt{\ell(x,y)}\le \sqrt{\frac{\beta}{\alpha}}\rbr{\ell(x,z)+\ell(z,y)}.
    \]
\end{proof}
\begin{proof}[Proof of Proposition \ref{prop:f*is_expectation}]
    By definition, for the optimal predictor $f^*$, we have
\[
f^*(x_i)=\argmin_{z}\EE_{y_i}[D_{\phi}(y_i,z)|x_i].
\]
Computing the gradient of RHS, we have that 
\[d(\EE_{y_i}[D_{\phi}(y_i,z)|x_i])/dz=\EE_{y_i}[dD_{\phi}(y_i,z)/dz|x_i]=\EE_{y_i}[\phi''(z)(z-y_i)|x_i]=\phi''(z)\EE_{y_i}[(z-y_i)|x_i].\] 
Since $\phi''$ is strictly positive, by the first-order condition, the optimal predictor is
\[
z^*_i=f^*(x_i)=\EE[y_i|x_i].
\]
In the random design setting, the claim can be proved similarly by replacing $x_i$ with any $x\in\cX$.
\end{proof}

\section{Proofs in \cref{sec:theory_guarantee}}\label{app:proofs of fixed and random design}
\subsection{Proof of Lemma \ref{lemma:W_n<Opt}}
In this subsection, we prove the key Lemma \ref{lemma:W_n<Opt}.
\begin{proof}[Proof of Lemma \ref{lemma:W_n<Opt}]
   Recall the definition that 
\[
    f^{\diamond}_{\rho}\in\argmin_{f\in\cC}\frac{1}{n}\sum_{i=1}^{n}\ell(y^{\diamond}_{i},f(x_i)).
    \]
By the three point equality, we have that
\[
D_\phi(f(x_i),\hat{f}(x_i))+D_{\phi}(\hat{f}(x_i),y^\diamond_i)+\rbr{\phi'(\hat{f}(x_i))-\phi'(y_i^\diamond)}\rbr{f(x_i)-\hat{f}(x_i)}=D_\phi(f(x_i),y_i^\diamond).
\]
Since we have $\alpha(x-y)^2\le D_\phi(x,y)\le \frac{\beta}{2}(x-y)^2$.
Thus, we know that $D_\phi(x,y)\ge \frac{\alpha}{2\beta} D_{\phi}(y,x)$. Therefore,
\[
D_\phi(f(x_i),\hat{f}(x_i))+D_{\phi}(\hat{f}(x_i),y^\diamond_i)+\rbr{\phi'(\hat{f}(x_i))-\phi'(y_i^\diamond)}\rbr{f(x_i)-\hat{f}(x_i)}\ge \frac{\alpha}{2\beta}D_\phi(f(x_i),y_i^\diamond).
\]

Recall that we set $\phi'(y_i^\diamond)=\phi'(\hat{f}(x_i))-\rho\varepsilon\tilde{w}_i$, then we have
\[
D_\phi(f(x_i),\hat{f}(x_i))+D_{\phi}(\hat{f}(x_i),y_i^\diamond)-\rho\rbr{\hat{f}(x_i)-f(x_i)}\varepsilon_i\tilde{w}_i\ge \frac{\alpha}{2\beta}D_\phi(y_i^\diamond,f(x_i)).
\]
Since the function $\phi$ is strongly convex, then we know that $\phi'$ is a monotonically increasing function on $\RR$. Then, we have the following four cases:

1) $\varepsilon_i\tilde{w}_i\ge 0$ and $\hat{f}(x_i)-f(x_i)\ge 0$: 

Then, by the Lipschitz continuity of $\phi'$ and its monotonicity, we have $\phi'(\hat{f}(x_i))-\phi'(f(x_i))\le \beta(\hat{f}(x_i)-f(x_i)),$ so
\[
\frac{\alpha}{2\beta}D_{\phi}(y_i^\diamond,f(x_i))\le D_\phi(f(x_i),\hat{f}(x_i))+D_{\phi}(\hat{f}(x_i),y_i^\diamond)-\frac{\rho}{\beta}[\phi'(\hat{f}(x_i))-\phi'(f(x_i))]\cdot\varepsilon_i\tilde{w}_i.
\]

2) $\varepsilon_i\tilde{w}_i\ge 0$ and $\hat{f}(x_i)-f(x_i)< 0$:

By the strong convexity of $\phi$, we have $(\phi'(x)-\phi'(y))(x-y)\ge \alpha(x-y)^2$, then we know that\[
\phi'(f(x_i))-\phi'(\hat{f}(x_i))\ge \alpha(f(x_i)-\hat{f}(x_i)).
\]
We plug this bound in and get
\[
\frac{\alpha}{2\beta}D_\phi(y_i^\diamond,f(x_i))\le D_\phi(f(x_i),\hat{f}(x_i))+D_{\phi}(\hat{f}(x_i),y_i^\diamond)-\frac{\rho}{\alpha}[\phi'(\hat{f}(x_i))-\phi'(f(x_i))]\cdot\varepsilon_i\tilde{w}_i.
\]
3) $\varepsilon_i\tilde{w}_i< 0$ and $\hat{f}(x_i)-f(x_i)\ge 0$, by similar argument as in 2), we have
\[
\frac{\alpha}{2\beta}D_\phi(y_i^\diamond,f(x_i))\le D_\phi(f(x_i),\hat{f}(x_i))+D_{\phi}(\hat{f}(x_i),y_i^\diamond)-\frac{\rho}{\alpha}[\phi'(\hat{f}(x_i))-\phi'(f(x_i))]\cdot\varepsilon_i\tilde{w}_i.
\]

4) $\varepsilon_i\tilde{w}_i< 0$ and $\hat{f}(x_i)-f(x_i)< 0$, by similar argument as in 1), we have that
\[
\frac{\alpha}{2\beta}D_{\phi}(y_i^\diamond,f(x_i))\le D_\phi(f(x_i),\hat{f}(x_i))+D_{\phi}(\hat{f}(x_i),y_i^\diamond)-\frac{\rho}{\beta}[\phi'(\hat{f}(x_i))-\phi'(f(x_i))]\cdot\varepsilon_i\tilde{w}_i.
\]
Overall, we either have $$
\frac{\alpha}{2\beta}D_{\phi}(y_i^\diamond,f(x_i))\le D_\phi(f(x_i),\hat{f}(x_i))+D_{\phi}(\hat{f}(x_i),y_i^\diamond)-\frac{\rho}{\beta}[\phi'(\hat{f}(x_i))-\phi'(f(x_i))]\cdot\varepsilon_i\tilde{w}_i,$$ or

$$\frac{\alpha}{2\beta}D_{\phi}(y_i^\diamond,f(x_i))\le D_\phi(f(x_i),\hat{f}(x_i))+D_{\phi}(\hat{f}(x_i),y_i^\diamond)-\frac{\rho}{\alpha}[\phi'(\hat{f}(x_i))-\phi'(f(x_i))]\cdot\varepsilon_i\tilde{w}_i.$$

Thus, no matter which one is true, we always have
\[
[\phi'(\hat{f}(x_i))-\phi'(f(x_i))]\cdot\varepsilon_i\tilde{w}_i\le \frac{\beta}{\rho}\rbr{D_\phi(f(x_i),\hat{f}(x_i))+D_{\phi}(\hat{f}(x_i),y_i^\diamond)}-\frac{\alpha^2}{\beta\rho}D_\phi(y_i^\diamond,f(x_i)).
\]
Summing over $i=1:n$, we have
\begin{align}\label{equa:D1}
\frac{\alpha^2}{\beta\rho}\rbr{\frac{1}{n}\sum_{i=1}^{n}\ell(y_i^\diamond,f(x_i))}\le\frac{\beta}{\rho}\sum_{i=1}^{n}\frac{1}{n}\rbr{\ell(f(x_i),\hat{f}(x_i))+D_{\phi}(\hat{f}(x_i),y_i^\diamond)}-\frac{1}{n}\sum_{i=1}^{n}[\phi'(\hat{f}(x_i))-\phi'(f(x_i))]\cdot\varepsilon_i\tilde{w}_i.
\end{align}

By some algebra and a shell argument, we have that
    \begin{align*}
        &\frac{\beta}{\rho}\frac{1}{n}\sum_{i=1}^{n}\ell(\hat{f}(x_i),f^\diamond_{\rho}(x_i))-\frac{1}{n}\sum_{i=1}^{n}[\phi'(\hat{f}(x_i))-\phi'(f_\rho^\diamond(x_i))]\cdot\varepsilon_i\tilde{w}_i+\frac{\beta}{\rho n}\sum_{i=1}^{n}D_{\phi}(\hat{f}(x_i),y_i^\diamond)\\
        \ge&\min_{f\in\cC}\cbr{\frac{\beta}{\rho n}\sum_{i=1}^{n}\ell(\hat{f}(x_i),f(x_i))-\frac{1}{n}\sum_{i=1}^{n}[\phi'(\hat{f}(x_i))-\phi'(f(x_i))]\cdot\varepsilon_i\tilde{w}_i}+\frac{\beta}{\rho n}\sum_{i=1}^{n}D_{\phi}(\hat{f}(x_i),y_i^\diamond)\\
        =&\min_{r\ge 0}\min_{f\in\cC(r)}\cbr{\frac{\beta}{\rho}r^2-\frac{1}{n}\sum_{i=1}^{n}[\phi'(\hat{f}(x_i))-\phi'(f(x_i))]\cdot\varepsilon_i\tilde{w}_i}+\frac{\beta}{\rho n}\sum_{i=1}^{n}D_{\phi}(\hat{f}(x_i),y_i^\diamond)\\
        =&\min_{r\ge 0}\cbr{\frac{\beta}{\rho}r^2-\sup_{f\in\cC(r)}\frac{1}{n}\sum_{i=1}^{n}[\phi'(\hat{f}(x_i))-\phi'(f(x_i))]\cdot\varepsilon_i\tilde{w}_i}+\frac{\beta}{\rho n}\sum_{i=1}^{n}D_{\phi}(\hat{f}(x_i),y_i^\diamond)\\
        =&\min_{r\ge 0}\cbr{\frac{\beta}{\rho}r^2- W_n(r)}+\frac{\beta}{\rho n}\sum_{i=1}^{n}D_{\phi}(\hat{f}(x_i),y_i^\diamond).\\
    \end{align*}
On the other hand, we take the minimization over both sides of \cref{equa:D1} and notice that $f_\rho^\diamond\in\argmin_{f\in\cC}\frac{1}{n}\sum_{i=1}^{n}\ell(y_i^\diamond,f(x_i))$ to get
\begin{align*}
      &\frac{\beta}{\rho n}\sum_{i=1}^{n}\ell(\hat{f}(x_i),f^{\diamond}_{\rho}(x_i))- W_n(\frac{1}{n}\sum_{i=1}^{n}\ell(\hat{f}(x_i),f^{\diamond}_{\rho}(x_i)))+\frac{\beta}{\rho n}\sum_{i=1}^{n}D_{\phi}(\hat{f}(x_i),y_i^\diamond)\\
      \ge&\min_{r\ge 0}\cbr{\frac{\beta}{\rho}r^2- W_n(r)}+\frac{\beta}{\rho n}\sum_{i=1}^{n}D_{\phi}(\hat{f}(x_i),y_i^\diamond)\\
      \ge& \frac{\alpha^2}{\beta\rho n}\sum_{i=1}^{n}\ell(y^{\diamond}_i,f^{\diamond}_{\rho}(x_i)).
  \end{align*}
Finally, we have that
\[
W_n(\frac{1}{n}\sum_{i=1}^{n}\ell(\hat{f}(x_i),f^{\diamond}_{\rho}(x_i)))\le \frac{\beta}{\rho n}\sum_{i=1}^{n}\ell(\hat{f}(x_i),f^{\diamond}_{\rho}(x_i))+\frac{\beta}{\rho n}\sum_{i=1}^{n}D_{\phi}(\hat{f}(x_i),y_i^\diamond)-\frac{\alpha^2}{\beta\rho n}\sum_{i=1}^{n}\ell(y^{\diamond}_i,f^{\diamond}_{\rho}(x_i)).
\]
Therefore, we finish the proof.

\end{proof}

\subsection{Proof of \Cref{thm:main}}\label{subsec:proof of mainthm}
In this subsection, we provide proofs of the main theoretical guarantees. Our proof relies on the following lemmas, the proofs of which are deferred to \Cref{app:Proofs of lemmas for mainthm}. We first briefly review our target. We are trying to bound $\frac{1}{n}\sum_{i=1}^{n}\ell(f^*(x_i),\hat{f}(x_i))$. By the convexity concerning the first variable, we have
\[
\frac{1}{n}\sum_{i=1}^{n}\ell(f^*(x_i),\hat{f}(x_i))\le \frac{1}{n}\sum_{i=1}^{n}\ell(y_i,\hat{f}(x_i))-\frac{1}{n}\sum_{i=1}^{n}\ell'_1(f^*(x_i),\hat{f}(x_i))w_i,
\]
where we use the fact that the true data generating process is $y_i=f^*(x_i)+w_i.$ Therefore, we only need to bound the term $|\text{Opt}^*(\hat{f})|+\EE_{w_{1:n}}[\Opt^*(\hat{f})]$, where
\[
\text{Opt}^*(\hat{f}):=\frac{1}{n}\sum_{i=1}^{n}\ell'_1(f^*(x_i),\hat{f}(x_i))w_i.
\]
Recall the quantities that we have defined in \cref{sec:theory_guarantee}:
\begin{equation}
W_n(r)=\sup_{f\in \cB_r(\hat{f})}\cbr{\frac{1}{n}\sum_{i=1}^{n}\ell_1'(\hat{f}(x_i),f(x_i))\cdot(\varepsilon_i\tilde{w}_i)};    
\end{equation}

\begin{equation}
 \widetilde{\Opt}^{\diamond}(f^{\diamond}_{\rho})=\frac{\beta}{n\rho}\sum_{i=1}^{n}\ell(\hat{f}(x_i),f^{\diamond}_{\rho}(x_i))+\frac{\beta}{\rho n}\sum_{i=1}^{n}D_{\phi}(\hat{f}(x_i),y_i^\diamond)-\frac{\alpha}{n\rho\beta}\sum_{i=1}^{n}\ell(y^{\diamond}_i,f^{\diamond}_{\rho}(x_i));   
\end{equation}
\begin{equation}
  f^\dagger=\argmin_{f\in\cF}\cbr{\frac{1}{n}\sum_{i=1}^{n}\ell(f^*(x_i),f(x_i))+\cR(f)};  
\end{equation}
 \begin{equation}
 \text{Opt}^\dagger(\hat{f})=\frac{1}{n}\sum_{i=1}^{n}\ell'_1(f^\dagger(x_i),\hat{f}(x_i))w_i;    
 \end{equation}
 and
 \begin{equation}
 Z_n^{\varepsilon}(r):=\sup_{f\in\cB_r(f^\dagger)}\sbr{\frac{1}{n}\sum_{i=1}^{n}\ell'_1(f^\dagger(x_i),f(x_i))\varepsilon_iw_i}.
 \end{equation}
When we say $\cB_r(f_0)$, we mean that the average empirical loss $\frac{1}{n}\sum_{i=1}^{n}\ell(f_0(x_i),f(x_i))\le r$. We abbreviate loss $\frac{1}{n}\sum_{i=1}^{n}\ell(f_1(x_i),f_2(x_i))$ as $L_n(f_1,f_2)$. Moreover, we should notice that $Z_n^{\varepsilon}(r)\ge 0$ since $f^\dagger\in\cB_r(f^\dagger)$.

%\begin{lemma}\label{lemma:E[Opt^*]<Opt}
 %   For any $t>0$, with probability at least $1-e^{-t^2}$, for any $r\ge \hat{r}_n$,
  %  \[
   % \EE_{y_{1:n}}[\Opt^*(\hat{f})]\le \Opt^*(\hat{f})+.
    %\]
%\end{lemma}
\begin{lemma}\label{lemma:Opt*<Optdagger}
   For any $t>0$, we have that with probability at least $1-2e^{-t^2}$,
    \[
    |\Opt^*(\hat{f})|\le |\Opt^\dagger(\hat{f})|+\sqrt{L_n(f^*,f^\dagger)}\frac{2||w||_{\infty}\beta^{3/2}t}{\alpha\sqrt{n}}.
    \]
Furthermore, for any $r\ge \sqrt{L_n(f^\dagger,\hat{f})}$, and any $t>0$, we have that with probability at least $1-4e^{-t^2}$,
\[
\max\cbr{|\Opt^\dagger(\hat{f})|,Z_n^\varepsilon(r)}\le \EE_{\varepsilon}[Z_n^\varepsilon(r)]+r\frac{2||w||_{\infty}\beta^{3/2}t}{\alpha\sqrt{n}}.
\]
\end{lemma}

By Lemma \ref{lemma:Opt*<Optdagger}, we could see that if $r\ge \sqrt{L_n(f^\dagger,\hat{f})}$, we have the upper bound that with probability at least $1-6e^{-t^2}$,
\[
|\Opt^*(\hat{f})|\le \EE_{\varepsilon}[Z_n^{\varepsilon}(r)]+\rbr{r+\sqrt{L_n(f^*,f^\dagger)}}\frac{2||w||_{\infty}\beta^{3/2}t}{\alpha\sqrt{n}}.
\]

Our next goal is to connect the term $\EE_{\varepsilon}[Z_n^{\varepsilon}(r)]$ with the term $\widetilde{\Opt}^\diamond(\hat{f})$ defined in Lemma \ref{lemma:W_n<Opt}. In order to do so, we define the following non-negative intermediate empirical process:
\[
\widetilde{W}_n(r):=\sup_{f\in\cB_r(\hat{f})}\frac{1}{n}\sum_{i=1}^{n}\ell_1'(\hat{f}(x_i),f(x_i))\cdot\varepsilon_iw_i.
\]
Remember that the noiseless estimation error is $\hat{r}_n:=\sqrt{\frac{1}{n}\sum_{i=1}^{n}\ell(f^\dagger(x_i),\hat{f}(x_i))}=\sqrt{L_n(f^{\dagger},\hat{f})}$. We have the following lemma.
\begin{lemma}\label{lemma:EZ<EW}
    For any $r\ge \hat{r}_n$, we have the bound
    \[
    \EE_{\varepsilon}[Z_n^{\varepsilon}(r)]\le\EE_{\varepsilon}[\widetilde{W}_n((2\beta/\alpha)^{1/2}(r+\hat{r}_n))]\le \EE_{\varepsilon}[\widetilde{W}_n(2(2\beta/\alpha)^{1/2}r)].
    \]
    Also, for the intermediate empirical process $\widetilde{W}_n(r)$, we have the bound with probability at least $1-2e^{-t^2}$ for any $r$,
    \[
    \abr{\EE_{\varepsilon}[\widetilde{W}_n(r)]-\widetilde{W}_n(r)}\le r\frac{2||w||_{\infty}\beta^{3/2}t}{\alpha\sqrt{n}}.
    \]
\end{lemma}
With Lemma \ref{lemma:EZ<EW}, we know that for $r\ge\hat{r}_n$, we have that with probability at least $1-2e^{-t^2}$,
\[
\EE_{\varepsilon}[Z_n^{\varepsilon}(r)]\le \widetilde{W}_n(3(\beta/\alpha)^{1/2}r)+ r\frac{6||w||_{\infty}\beta^2t}{\alpha^{3/2}\sqrt{n}}.
\]
Therefore, combining all these parts together, we have shown that with probability at least $1-8e^{-t^2}$,
\begin{align*}
    \Opt^*(\hat{f})\le \widetilde{W}_n(3(\beta/\alpha)^{1/2}r)+\cbr{\sqrt{L_n(f^*,f^\dagger)}+5r}\frac{2||w||_{\infty}(\beta^{3/2}\vee\beta^2)t}{(\alpha^{3/2}\wedge\alpha)\sqrt{n}}.
\end{align*}
Finally, we give our lemma about linking $\widetilde{W}_n(2r)$ with $\widetilde{\Opt}^{\diamond}(f^{\diamond}_{\rho})$, whose proof relies on the key Lemma \ref{lemma:W_n<Opt}.
\begin{lemma}\label{lemma:tildeW_n<Opt_n+An}
    For any radius $r$, we have 
    \[
    \widetilde{W}_n(3(\beta/\alpha)^{1/2}r)\le \widetilde{\Opt}^{\diamond}(f^{\diamond}_{\rho})+A_n(\hat{f}),
    \]
    where $f^\diamond_{\rho}$ is the wild solution with $L_n(\hat{f},f^\diamond_\rho)=3(\beta/\alpha)^{1/2}r$.
\end{lemma}

Plugging the inequality in Lemma \ref{lemma:tildeW_n<Opt_n+An} back, we have that with probability at least $1-8e^{-t^2}$,
\[
\Opt^*(\hat{f})\le\widetilde{\Opt}^{\diamond}(f^{\diamond}_{\rho})+A_n(\hat{f})+\cbr{\sqrt{L_n(f^*,f^\dagger)}+5r}\frac{2||w||_{\infty}(\beta^{3/2}\vee\beta^2)t}{(\alpha^{3/2}\wedge\alpha)\sqrt{n}}.
\]
By a change of variable $\delta\leftarrow e^{-t^2}$, we therefore finish the proof of \Cref{thm:main}.
\section{Proof in \Cref{sec:bounding_hatr_n}}\label{app:proof of bound_hat_rn}
In this section, we prove \Cref{thm:bound_hat_rn}.
%\section{Proof of \Cref{thm:bound_hat_rn}}
We first introduce the following three lemmas. Their proofs are deferred to \cref{app:proofs of lemmas for bound_hat_rn}.
\begin{lemma}\label{lemma:hat_rn<Z_n}
    Given the procedure $\Alg$ that is $\phi$-non-expansive around $f^*$, if $\hat{r}_n=\sqrt{L_n(f^\dagger,\hat{f})}$, then we have \[
    \hat{r}_n^2\le Z_n(\hat{r}_n),
    \]
    where $Z_n(r):r\mapsto\sup_{f\in\cB_r(f^\dagger)}\sbr{\frac{1}{n}\sum_{i=1}^{n}\ell_1'(f^\dagger(x_i),f(x_i))w_i}$ is an empirical process.
\end{lemma}

The random variable $Z_n(r)$ could be viewed as a realization of $Z_n^{\varepsilon}(r)$ since we assume that the distribution of the noise is symmetric. Then, we have the following second lemma.
\begin{lemma}\label{lemma:peeling}
    For any $t\ge 3$, with probability at least $1-2e^{-t^2}$,
    \[
    Z_n(r)\le \EE_{\varepsilon}[Z_n^{\varepsilon}(1+\frac{1}{t})r]+r^2\frac{4||w||_{\infty}\beta^{3/2}}{\alpha t},
    \]
    uniformly for $r\ge \frac{t^2}{\sqrt{n}}$.
\end{lemma}
We now proceed with our proof about \Cref{thm:bound_hat_rn}. Fixing any $t\ge 3$, we either have $\hat{r}_n\le \frac{t^2}{\sqrt{n}}$ where our claim already holds, or we have $\hat{r}_n> \frac{t^2}{\sqrt{n}}$. Under the latter case, by Lemma \ref{lemma:hat_rn<Z_n} and Lemma \ref{lemma:peeling}, we have
\[
\hat{r}_n^2\le Z_n(\hat{r}_n)\le \EE_{\varepsilon}[Z_n^{\varepsilon}(1+\frac{1}{t})\hat{r}_n]+\frac{4||w||_{\infty}\beta^{3/2}}{\alpha t}\hat{r}_n^2.
\]
We apply the first claim from Lemma \ref{lemma:EZ<EW} to get
\[
\hat{r}_n^2\le \EE_{\varepsilon}\sbr{\widetilde{W}_n\rbr{(2+\frac{1}{t})\hat{r}_n}}+\frac{4||w||_{\infty}\beta^{3/2}}{\alpha t}\hat{r}_n^2.
\]
Moreover, by the second claim of Lemma \ref{lemma:EZ<EW}, setting $s=\frac{\hat{r}_n\sqrt{n}}{t}$, we have that with probability at least $1-2e^{-s^2}$,
\[
\EE_{\varepsilon}\sbr{\widetilde{W}_n(2+\frac{1}{t})\hat{r}_n}\le \widetilde{W}_n\rbr{(2+\frac{1}{t})\hat{r}_n}+\hat{r}_n^2\frac{2||w||_{\infty}\beta^{3/2}}{\alpha t}.
\]
By some algebra,
$s^2=\frac{\hat{r}_n^2n}{t^2}\ge t^2$, where we use the assumption that $\hat{r}_n\ge \frac{s^2}{\sqrt{n}}$. Consequently, with a probability of at least $1-2e^{-t^2}$, 
\[
\EE_{\varepsilon}[\widetilde{W}_n(\rbr{(2+1/t)}\hat{r}_n)]\le \widetilde{W}_n((2+\frac{1}{t})\hat{r}_n)+\hat{r}_n^2\frac{2||w||_{\infty}\beta^{3/2}}{\alpha t}.
\]
Combining these two parts together, with probability at least $1-4e^{-t^2}$,
\[
\hat{r}_n^2 \le \widetilde{W}_n((2+\frac{1}{t})\hat{r}_n)+\hat{r}_n^2\frac{6||w||_{\infty}\beta^{3/2}}{\alpha t}\le W_n((2+1/t)\hat{r}_n)+\hat{r}_n^2\frac{6||w||_{\infty}\beta^{3/2}}{\alpha t}+A_n(\hat{f}).
\]
The last inequality follows from the proof of Lemma \ref{lemma:tildeW_n<Opt_n+An}. Combining the additional term of the case where $\hat{r}_n\le \frac{t^2}{\sqrt{n}}$, we finish the proof of the first claim.

To prove the second claim, we first require the following lemma.
\begin{lemma}\label{lemma:scale_concave_Wn}
    For any $v\ge u>0$, when $C_0=\sqrt{\beta/\alpha}$, we have that
    \[
    \frac{W_n(v)}{v}\le \frac{W_n(C_0u)}{u}.
    \]
\end{lemma}

Now we prove the second bound of \Cref{thm:bound_hat_rn}. We either have $\hat{r}_n\le r^\diamond_{\rho}$ or $\hat{r}_n> r^\diamond_{\rho}$. In the latter case, we write
\begin{align*}
    W_n([2+1/t]\hat{r}_n)=[2+1/t]\hat{r}_n\frac{ W_n([2+1/t]\hat{r}_n)}{[2+1/t]\hat{r}_n}\le [2+1/t]\hat{r}_n\frac{ W_n(C_0[2+1/t]r^\diamond_\rho)}{[2+1/t]r^\diamond_\rho}=\hat{r}_n \frac{ W_n(C_0[2+1/t]r^\diamond_\rho)}{r^\diamond_\rho}.
\end{align*}
Plugging this back to the first claim of \Cref{thm:bound_hat_rn} and applying the change of variable $\delta\leftarrow e^{-t^2}$, we shall finish the proof.

\section{Proofs in \cref{subsec:proof of mainthm}}\label{app:Proofs of lemmas for mainthm}
\iffalse
\begin{proof}[Proof of Lemma \ref{lemma:E[Opt^*]<Opt}]
    By the symmetric distribution assumption, we can write $w_i$ as $\varepsilon_i|w_i|$. We define the function $Q(\varepsilon):=\frac{1}{n}\sum_{i=1}^{n}\ell_1'(f^*(x_i),\hat{f}(x_i))\varepsilon_i|w_i|$, and argue that this function is Lipschitz continuous.
    \begin{align*}
        Q(\varepsilon)-Q(\varepsilon')&=\frac{1}{n}\sum_{i=1}^{n}\ell_1'(f^*(x_i),\hat{f}(x_i))|w_i|(\varepsilon_i-\varepsilon_i')\\
        &\le \frac{1}{n}\sum_{i=1}^{n}|\ell_1'(f^*(x_i),\hat{f}(x_i))|w_i||\varepsilon_i-\varepsilon_i'|\\
        &\le ||w||_{\infty}\frac{\beta r}{\alpha\sqrt{n}}\|\varepsilon-\varepsilon'\|_2.
    \end{align*}
    Therefore, we could apply Lemma \ref{lemma:concentration_Lip} to get that with probability at least $1-e^{-t^2}$,
    \[
    \EE_{\varepsilon_{1:n}}[Q(\varepsilon)]\le \frac{\sqrt{2}\|w\|_{\infty} \beta^{3/2} r}{\alpha\sqrt{n}}t
    \]
\end{proof}
\fi
\begin{proof}[Proof of Lemma \ref{lemma:Opt*<Optdagger}]
By some algebra, we have that
\begin{align*}
        \Opt^*(\hat{f})&=\frac{1}{n}\sum_{i=1}^{n}\ell_1'(f^*(x_i)-f^\dagger(x_i)+f^\dagger(x_i),\hat{f}(x_i))w_i\\
        &=\frac{1}{n}\sum_{i=1}^{n}\ell_1'(f^\dagger(x_i),\hat{f}(x_i))w_i+\frac{1}{n}\sum_{i=1}^{n}\sbr{\ell_1'(f^*(x_i),\hat{f}(x_i))-\ell_1'
        (f^\dagger(x_i),\hat{f}(x_i))}w_i\\
        &=\Opt^\dagger(\hat{f})+\frac{1}{n}\sum_{i=1}^{n}\sbr{\ell_1'(f^*(x_i),\hat{f}(x_i))-\ell_1'(f^\dagger(x_i),\hat{f}(x_i))}w_i\\
        &=\Opt^\dagger(\hat{f})+\frac{1}{n}\sum_{i=1}^{n}\sbr{\phi'(f^*(x_i))-\phi'(f^\dagger(x_i))}w_i
    \end{align*}
    %By definition, and a Taylor expansion, we have that

    Now we analyze the term $\frac{1}{n}\sum_{i=1}^{n}\sbr{\phi'(f^*(x_i))-\phi'(f^\dagger(x_i))}w_i$. We assume that every $w_i$ has an independent symmetric distribution, i.e., $$w_{i}=\varepsilon_{i}|w_{i}|,\ w_{i}\perp w_{k},\ \forall\ i,k=1,\cdots,n.$$ We denote $\Bar{w}_i$ as the amplitude, and $\varepsilon_i$ as the random variable denoting the sign of $w_i$. Hence, we have that $w_i=\varepsilon_i\cdot\Bar{w}_i$. Moreover, we assume that the compact set $\cC$ has diameter $D_0$.
    Then, we have
    \begin{align*}
        &\frac{1}{n}\sum_{i=1}^{n}\sbr{\phi'(f^*(x_i))-\phi'(f^\dagger(x_i))}w_i\\
        =&\frac{1}{n}\sum_{i=1}^{n}\sbr{\phi'(f^*(x_i))-\phi'(f^\dagger(x_i))}\varepsilon_i\Bar{w}_i\\
        =&\frac{1}{n}\sum_{i=1}^{n}\sbr{(\phi'(f^*(x_i))-\phi'(f^\dagger(x_i)))\cdot \Bar{w}_i}\varepsilon_i
    \end{align*}
    Conditioning on the amplitude vector $\Bar{w}=(\Bar{w}_1,\cdots,\Bar{w}_n)$, we define the function, 
    \[
    (\varepsilon_1,\cdots,\varepsilon_n)\mapsto G(\varepsilon):=\frac{1}{n}\sum_{i=1}^{n}\sbr{(\phi'(f^*(x_i))-\phi'(f^\dagger(x_i)))\cdot\Bar{w}_i}\varepsilon_i,
    \]
    Then we have that,
    \begin{align*}
    \abr{ G(\varepsilon)-G(\varepsilon')}&=\abr{\frac{1}{n}\sum_{i=1}^{n}\sbr{(\phi'(f^*(x_i))-\phi'(f^\dagger(x_i)))\cdot\Bar{w}_i}{(\varepsilon_i-\varepsilon_i')}}\\
    &\le \frac{1}{n}\sum_{i=1}^{n}\abr{\sbr{\phi'(f^*(x_i))-\phi'(f^\dagger(x_i))}\cdot\Bar{w}_i}|\varepsilon_i-\varepsilon_i'|\\
    &\le \frac{1}{n}\sum_{i=1}^{n}|\Bar{w}_i|\abr{\sbr{\phi'(f^*(x_i))-\phi'(f^\dagger(x_i))}}\varepsilon_i-\varepsilon_i'|\\
    &\le\frac{||w||_{\infty}}{n}\sum_{i=1}^{n}\beta |f^*(x_i)-f^\dagger(x_i)|\cdot|\varepsilon_i-\varepsilon_i'| \\
    &\le \frac{||w||_{\infty}\beta}{n}\rbr{\sum_{i=1}^{n}(f^*(x_i)-f^\dagger(x_i))^2}^{1/2}\rbr{\sum_{i=1}^{n}(\varepsilon_i-\varepsilon_i')^2}^{1/2}\\
    \end{align*}
    Notice that $(\sum_{i=1}^{n}(\varepsilon_i-\varepsilon_i')^2)^{1/2}=||\varepsilon-\varepsilon'||_{2}$. And now we focus on the coefficient term. By Proposition \ref{prop:PL_ineq} and Lemma \ref{lemma:PL_ineq}, we know that
\begin{align*}
    \sum_{i=1}^{n}||f^*(x_i)-f^\dagger(x_i)||_2^2\le& \sum_{i=1}^{n}\frac{2\beta}{\alpha^2}(D_{\phi}(f^*(x_i),f^\dagger(x_i))-D_{\phi}(f^\dagger(x_i),f^\dagger(x_i)))\\
    =&\frac{2\beta}{\alpha^2}\sum_{i=1}^{n}D_{\phi}(f^*(x_i),f^\dagger(x_i)).
\end{align*}
Therefore, we have that
\[
\rbr{\sum_{i=1}^{n}||f^*(x_i)-f^\dagger(x_i)||_2^2}^{1/2}\le\frac{\sqrt{2\beta}}{\alpha}\rbr{\sum_{i=1}^{n}D_{\phi}(f^*(x_i),f^\dagger(x_i))}^{1/2}=\frac{\sqrt{2\beta}}{\alpha}\sqrt{n}\sqrt{L_n(f^*,f^\dagger)}.
\]
Plugging this back, we have that $G(\varepsilon)$ is Lipschitz continuous with constant $\frac{\sqrt{2}||w||_{\infty}\beta^{3/2}}{\alpha\sqrt{n}}\sqrt{L_n(f^*,f^\dagger)}$. Then, we could apply Lemma \ref{lemma:concentration_Lip} to get that with probability at least $1-2e^{-t^2}$,
\[
|G(\varepsilon)|\le \sqrt{2}\frac{\sqrt{2}||w||_{\infty}\beta^{3/2}}{\alpha\sqrt{n}}\sqrt{L_n(f^*,f^\dagger)}t=\frac{2||w||_{\infty}\beta^{3/2}t}{\alpha\sqrt{n}}\sqrt{L_n(f^*,f^\dagger)}.
\]

Now, for any $r\ge \sqrt{L_n(f^\dagger,f^*)}$, and any $t>0$, we define the empirical process, conditioning on the amplitude vectors $\cbr{\Bar{w}_i}_{i=1}^{n}$,
\[
H(\varepsilon):=Z_n^{\varepsilon}(r)=\sup_{f\in\cB_r(f^\dagger)}\frac{1}{n}\sum_{i=1}^{n}\ell_1'(f^\dagger(x_i),f(x_i))\varepsilon_i\cdot \bar{w}_i=\sup_{f\in\cB_r(f^\dagger)}\frac{1}{n}\sum_{i=1}^{n}\ell_1'(f^\dagger(x_i),f(x_i))\cdot w_i \varepsilon_i.
\]
We show that $H(\cdot)$ is also Lipschitz conditioned on the amptitude vector $\bar{w}$.
\begin{align*}
    &H(\varepsilon)-H(\varepsilon')\\
    \le &\sup_{f\in\cB_r(f^\dagger)}\frac{1}{n}\sum_{i=1}^{n}\ell_1'(f^\dagger(x_i),f(x_i))\cdot w_i(\varepsilon_i-\varepsilon_i')\\
    \le&||w||_{\infty}\sup_{f\in\cB_r(f^\dagger)}\frac{1}{n}\sum_{i=1}^{n}\ell_1'(f^\dagger(x_i),f(x_i))(\varepsilon_i-\varepsilon_i')\\
    \le& ||w||_{\infty}\sup_{f\in\cB_r(f^\dagger)}\frac{1}{n}\sum_{i=1}^{n}|\phi'(f^\dagger(x_i))-\phi'(f(x_i))|\cdot|\varepsilon_i-\varepsilon_i'|\\
    \le&||w||_{\infty}\sqrt{d}\beta\sup_{f\in\cB_r(f^\dagger)}\frac{1}{n}\sum_{i=1}^{n}|f^\dagger(x_i)-f(x_i)|\cdot |\varepsilon_i-\varepsilon_i'|\\
    \le& ||w||_{\infty}\beta\sup_{f\in\cB_r(f^\dagger)}\frac{1}{n}\rbr{\sum_{i=1}^{n}(f^\dagger(x_i)-f(x_i))^2}^{1/2}\cdot||\varepsilon-\varepsilon'||_2.
\end{align*}
By reversing the role of $\varepsilon$ and $\varepsilon'$, we know that $H$ is Lipschitz. Then by the same argument, we know that the Lipschitz constant is $\frac{\sqrt{2}||w||_{\infty}r}{\alpha\sqrt{n}}$, thus with probability at least $1-2e^{-t^2}$,
\[
\abr{H(\varepsilon)-\EE_{\varepsilon}[H(\varepsilon)]}\le \frac{2||w||_{\infty}\beta^{3/2}t}{\alpha\sqrt{n}}r.
\]
Plugging this back and notice that when $r\ge \sqrt{L_n(f^\dagger,f^*)}$, we have $\hat{f}\in\cB_r(f^\dagger)$, we shall finish the proof.
\end{proof}
\begin{proof}[Proof of Lemma \ref{lemma:EZ<EW}]
    Recall the definition of $Z_n^{\varepsilon}(r)$, let $g$ be any function that achieves the supremum. Then,
    \begin{align*}
    Z_n^{\varepsilon}(r)&=\frac{1}{n}\sum_{i=1}^{n}\ell_1'(f^\dagger(x_i),g(x_i))\varepsilon_i\cdot w_i\\
    =&\frac{1}{n}\sum_{i=1}^{n}\ell_1'(\hat{f}(x_i),g(x_i))\varepsilon_i\cdot w_i+\frac{1}{n}\sum_{i=1}^{n}(\phi'(f^\dagger(x_i))-\phi'(\hat{f}(x_i))\varepsilon_i\cdot w_i\\
    \end{align*}
    Taking the expectation on both sides, we analyze these two terms one by one.

For the second term, we have that the expectation is $0$ because we could condition on $w_i$ and use the linearity of inner product and expectation to swap the order. $\varepsilon_i$ is independent of $w_i$ and hence $\hat{f}$. The expectation of $\varepsilon_i$ is $0$. 
    
    For the first one, by the condition $r\ge \hat{r}_n=\sqrt{L_n(f^{\dagger},\hat{f})}$ and Lemma \ref{lemma:triangle_Ln}, we have that
    \[
    L_n(g,\hat{f})\le \sqrt{\frac{2\beta}{\alpha}}\rbr{L_n(g,f^{\dagger})+L_n(f^\dagger,\hat{f})}\le\sqrt{\frac{2\beta}{\alpha}}(r+\hat{r}_n) \le 2\sqrt{\frac{2\beta}{\alpha}}r.
    \]
    This is by the fact that $g\in\cB_r(f^\dagger)$ and $r\ge\hat{r}_n$. Then denoting $c_1$ to be $\sqrt{\frac{2\beta}{\alpha}}$, we have
    \begin{align*}
    \frac{1}{n}\sum_{i=1}^{n}\ell_1'(\hat{f}(x_i),g(x_i))\varepsilon_i\cdot w_i&\le \sup_{f\in\cB_{c_1(r+\hat{r}_n)}(\hat{f})}\ell_1'(\hat{f}(x_i),f(x_i))\varepsilon_i\cdot w_i\\
    &\le \sup_{f\in\cB_{2c_1r}(\hat{f})}\ell_1'(\hat{f}(x_i),f(x_i))\varepsilon_i\cdot w_i,
    \end{align*}
    and we get the first bound of the lemma.
    \[
    \EE_{\varepsilon}[Z_n^{\varepsilon}(r)]\le\EE_{\varepsilon}[\widetilde{W}_n(c_1(r+\hat{r}_n)] \le \EE_{\varepsilon}[\widetilde{W}_n(2c_1r)].
    \]
    Moreover, similar to the process of analyzing $G(\varepsilon)$ and $H(\varepsilon)$, we have that with probability at least $1-2e^{-t^2}$,
    \[
    |\EE_{\varepsilon}[\widetilde{W}_n(r)]- \widetilde{W}_n(r)|\le r\frac{2||w||_{\infty}\beta^{3/2}t}{\alpha\sqrt{n}}.
    \]
    So we finish the proof.
\end{proof}
\begin{proof}[Proof of Lemma \ref{lemma:tildeW_n<Opt_n+An}]
    By the definition of the wild noise $\widetilde{w}_i=y_i-\hat{f}(x_i)$, we have that
    \[
    \varepsilon_i\cdot w_i=\varepsilon_i\cdot \tilde{w}_i+\varepsilon_i\cdot (\hat{f}(x_i)-f^*(x_i)).
    \]
    Therefore, we have that
    \begin{align*}
        \widetilde{W}_n(3(\beta/\alpha)^{1/2}r)=&\sup_{f\in\cB_{3(\beta/\alpha)^{1/2}r}(\hat{f})}\frac{1}{n}\sum_{i=1}^{n}\ell_1'(\hat{f}(x_i),f(x_i))\varepsilon_i\cdot w_i\\
        =&\sup_{f\in\cB_{3(\beta/\alpha)^{1/2}r}(\hat{f})}\frac{1}{n}\sum_{i=1}^{n}\ell_1'(\hat{f}(x_i),f(x_i))[\varepsilon_i\cdot \tilde{w}_i+\varepsilon_i\cdot (\hat{f}(x_i)-f^*(x_i))]\\
        \le& \sup_{f\in\cB_{3(\beta/\alpha)^{1/2}r}(\hat{f})}\frac{1}{n}\sum_{i=1}^{n}\ell_1'(\hat{f}(x_i),f(x_i))\varepsilon_i\cdot \tilde{w}_i\\
        +&\sup_{f\in\cB_{3(\beta/\alpha)^{1/2}r}(\hat{f})}\frac{1}{n}\sum_{i=1}^{n}\ell_1'(\hat{f}(x_i),f(x_i))[\varepsilon_i\cdot (\hat{f}(x_i)-f^*(x_i))]\\
        = & W_n(3(\beta/\alpha)^{1/2}r)+A_n(\hat{f}).
    \end{align*}
    Finally, by Lemma \ref{lemma:W_n<Opt}, we have that $ W_n(3(\beta/\alpha)^{1/2}r)\le \widetilde{\Opt}(f^\diamond_\rho)$ for the wild solution such that $L_n(\hat{f},f^{\diamond}_\rho)=3(\beta/\alpha)^{1/2}r$.
\end{proof}
\section{Proof of \Cref{app:proof of bound_hat_rn}}\label{app:proofs of lemmas for bound_hat_rn}
\begin{proof}[Proof of Lemma \ref{lemma:hat_rn<Z_n}]
 By definition \ref{def:non_expansive}, we have that
 \begin{align*}
     \hat{r}_n^2=L_n(f^\dagger,\hat{f})\le \frac{1}{n}\sum_{i=1}^{n}\ell_1'(f^\dagger(x_i),\hat{f}(x_i))w_i\le Z_n(\hat{r}_n).
 \end{align*}
 So we finish the proof.
\end{proof}
\begin{proof}[Proof of Lemma \ref{lemma:peeling}]
    Recall the argument in the proof of Lemma \ref{lemma:Opt*<Optdagger} such that $Z_n^{\varepsilon}(r)$ is Lipschitz continuous in $\varepsilon$ with constant $\frac{\sqrt{2}||w||_{\infty}r}{\alpha\sqrt{n}}$. Then, by Lemma \ref{lemma:concentration_Lip}, we have that for any $t>0$,
    \[
    \PP\rbr{Z_n(r))\ge \EE_{\varepsilon}[Z_n^{\varepsilon}(r)]+r^2\frac{2||w||_{\infty}\beta^{3/2}}{\alpha t}}\le 4e^{-\frac{nr^2}{t^2}}\le e^{-t^2},\ r\ge t^2/\sqrt{n}.
    \]
    We define $\cG$ to be the event that this inequality is violated for some $r\ge \frac{t^2}{\sqrt{n}}$. Our method is a peeling argument, which could also be found in the proofs in other papers \citep{xia2024predictionaidedsurrogatetraining,hu2025pre}. Specifically, we define the event:
    \[
    \cG_m:=\cbr{\exists r\in[(1+\frac{1}{t})^m\frac{t^2}{\sqrt{n}},(1+\frac{1}{t})^{m+1}\frac{t^2}{\sqrt{n}}),\ \text{the bound is violated}}.
    \]
    Now, we provide an upper bound of $\PP(\cG_m)$. If the event $\cG_m$ is true, we denote $p_m$ to be $(1+\frac{1}{t})^{m}\frac{t^2}{\sqrt{n}}$ and then have that
    \begin{align*}
        Z_n(p_{m+1})\ge Z_n(r)&\ge\EE_{\varepsilon}[Z_n((1+\frac{1}{t})r)]+r^2\frac{2||w||_{\infty}\beta^{3/2}}{\alpha t}\\
        &\ge \EE_{\varepsilon}[Z_n(p_{m+1})]+p_m^2\frac{2||w||_{\infty}\beta^{3/2}}{\alpha t}\\
        &\ge \EE_{\varepsilon}[Z_n(p_{m+1})]+p_{m+1}^2\frac{||w||_{\infty}\beta^{3/2}}{\alpha t}.
    \end{align*}
    Therefore, applying the same argument about $Z_n^{\varepsilon}(r)$ in the proof of Lemma \ref{lemma:Opt*<Optdagger} to get:
    \[
    \PP\rbr{\cG_m}\le \PP\rbr{Z_n(p_{m+1})\ge\EE_{\varepsilon}[Z_n(p_{m+1})]+p_{m+1}^2\frac{||w||_{\infty}\beta^{3/2}}{\alpha t}}\le 2e^{-\frac{n}{t^2}p_{m+1}^2}.
    \]
    Finally, by a union bound over $m=0,1,\cdots,$ we have
    \[
    \PP(\cG)=\PP(\cup_{m=1}^{\infty}\cG_m)\le \sum_{m=0}^{\infty}\PP(\cG_m)\le 2\sum_{m=0}^{\infty}e^{-\frac{n}{t^2}p_{m+1}^2}\le 2e^{-s^2}.
    \]
    So we finish the proof.
\end{proof}
\begin{proof}[Proof of Lemma \ref{lemma:scale_concave_Wn}]
    For any $s\ge t>0$ denoting $g_s$ and $g_t$ as the functions in the definition of $W_n(\cdot)$ that achieve the maxima of $W_n(s)$ and $W_n(t)$. For any fixed $a\in[0,1]$, we set $r:=as+(1-a)t$ and define $g_r:=ag_s+(1-a)g_t$. By the convexity of $\cF$, $g_r\in\cF$. Then we have
\begin{align*}
    &\sqrt{L_n(\hat{f},ag_s+(1-a)g_t)}\\
    \le& \sqrt{\beta \frac{1}{n}\sum_{i=1}^{n}\rbr{\hat{f}(x_i)-(ag_s+(1-a)g_t(x_i))}^2}\\
    =&\sqrt{\beta}\rbr{a||\hat{f}-g_s||_n+(1-a)||\hat{f}-g_t||_n}\\
    \le&\sqrt{\frac{\beta}{\alpha}}\rbr{a\sqrt{L_n(\hat{f},g_s)}+\sqrt{L_n(\hat{f},g_t)}}\\
    =&\sqrt{\frac{\beta}{\alpha}}(as+(1-a)t)\\
    =&\sqrt{\frac{\beta}{\alpha}}r.
\end{align*}
Consequently, $g_r$ is feasible for the supremum in $W_n(\sqrt{\frac{\beta}{\alpha}}r)$, so we have
\begin{align*}
    aW_n(s)+(1-a)W_n(t)&=\frac{1}{n}\sum_{i=1}^{n}\sbr{\phi'(\hat{f}(x_i))-\rbr{a\phi'(g_s(x_i))+(1-a)\phi'(g_t(x_i))}}\varepsilon_i\cdot \widetilde{w}_i\\
    \le& \sup_{\cB_{(\beta/\alpha)^{1/2}r}(\hat{f})}\sbr{\phi'(\hat{f}(x_i))-\phi'(f(x_i))}\varepsilon_i\cdot \widetilde{w}_i\\
    =&W_n((\beta/\alpha)^{1/2}r).
\end{align*}
In brief, we have proved that $aW_n(s)+(1-a)W_n(t)\le W_n((\beta/\alpha)^{1/2}(as+(1-a)t))$, $\forall\ a\in[0,1]$. Notice that $W_n(0)=0$, and denote $(\beta/\alpha)^{1/2}$ as $C_0$. Then, for any $0<u\le v$, setting $a=\frac{u}{v}$, we have
\[
W_n(C_0 u)\ge W_n(0)(1-\frac{u}{v})+W_n(v)\frac{u}{v}\Rightarrow \frac{W_n(v)}{v}\le \frac{W_n(C_0u)}{u}.
\]
We finish the proof.
\end{proof}
\end{document}